\documentclass[nonatbib]{article}
\usepackage[preprint]{neurips_2026}
\usepackage[numbers,sort&compress]{natbib}

\usepackage[utf8]{inputenc}
\usepackage[T1]{fontenc}
\usepackage{hyperref}
\usepackage{url}
\usepackage{booktabs}
\usepackage{amsfonts} 
\usepackage{amsmath}
\usepackage{amssymb}
\usepackage{amsthm}
\usepackage{mathtools}
\usepackage{nicefrac}
\usepackage{microtype}
\usepackage{xcolor}
\usepackage{graphicx}
\usepackage{subcaption}
\usepackage{multirow}
\usepackage{algorithm}
\usepackage{algpseudocode}
\usepackage{wrapfig}
\usepackage{bbm}
\usepackage{bm}

\usepackage[capitalize,noabbrev]{cleveref}

\usepackage{enumitem}
\usepackage{tikz}
\usetikzlibrary{shapes.geometric, arrows.meta, positioning, fit,
                decorations.pathreplacing, backgrounds, calc}

\newtheorem{theorem}{Theorem}
\newtheorem{proposition}{Proposition}
\newtheorem{lemma}{Lemma}
\newtheorem{corollary}{Corollary}
\newtheorem{definition}{Definition}
\newtheorem{remark}{Remark}
\newtheorem{assumption}{Assumption}

\newcommand{\G}{\mathcal{G}}
\newcommand{\V}{\mathcal{V}}
\newcommand{\E}{\mathcal{E}}
\newcommand{\X}{\mathbf{X}}
\newcommand{\A}{\mathbf{A}}
\newcommand{\Ahat}{\hat{\mathbf{A}}}

\newcommand{\Phitilde}{\tilde{\bm{\Phi}}}
\newcommand{\eps}{\varepsilon}
\newcommand{\ftheta}{f_{\theta}}
\newcommand{\R}{\mathbb{R}}
\newcommand{\edgeset}{\mathcal{E}}
\newcommand{\calA}{\mathcal{A}}
\newcommand{\calK}{\mathcal{K}}
\newcommand{\calM}{\mathcal{M}}
\newcommand{\calN}{\mathcal{N}}
\newcommand{\calL}{\mathcal{L}}
\newcommand{\calR}{\mathcal{R}}
\newcommand{\calG}{\mathcal{G}}
\newcommand{\norm}[1]{\left\lVert #1 \right\rVert}
\newcommand{\abs}[1]{\left\lvert #1 \right\rvert}

\newcommand{\ours}{\textsc{PrivX}}
\newcommand{\ourf}{\textsc{PrivF}}
\DeclareMathOperator*{\argmin}{arg\,min}

\tikzset{
  gnnbox/.style  = {rectangle, rounded corners=3pt, draw=teal!70!black,
                    fill=teal!15, text centered, font=\small,
                    minimum height=1.8em, minimum width=4.5em},
  dpbox/.style   = {rectangle, rounded corners=3pt, draw=blue!70!black,
                    fill=blue!12, text centered, font=\small,
                    minimum height=1.8em, minimum width=4.5em},
  diffbox/.style = {rectangle, rounded corners=3pt, draw=orange!80!black,
                    fill=orange!15, text centered, font=\small,
                    minimum height=1.8em, minimum width=4.5em},
  outbox/.style  = {rectangle, rounded corners=3pt, draw=green!60!black,
                    fill=green!12, text centered, font=\small,
                    minimum height=1.8em, minimum width=4.5em},
  atkbox/.style  = {rectangle, rounded corners=3pt, draw=red!70!black,
                    fill=red!12, text centered, font=\small,
                    minimum height=1.8em, minimum width=4.5em},
  graybox/.style = {rectangle, rounded corners=3pt, draw=gray!60,
                    fill=gray!10, text centered, font=\small,
                    minimum height=1.8em, minimum width=4.5em},
  arr/.style     = {-Stealth, thick},
  darr/.style    = {-Stealth, thick, dashed, blue!60!black},
}

\title{Graph Reconstruction from Differentially Private GNN Explanations}

\author{%
  Rishi Raj Sahoo$^{1,2}$ \quad Jyotirmaya Shivottam$^{1,2}$ \quad Subhankar Mishra$^{1,2}$ \\
  $^{1}$National Institute of Science Education and Research (NISER), Bhubaneswar, India \\
  $^{2}$Homi Bhabha National Institute, Mumbai, India \\
  \texttt{\{rishiraj.sahoo, jyotirmaya.shivottam, smishra\}@niser.ac.in}
}

\begin{document}
\maketitle

\begin{abstract}
Regulatory frameworks such as GDPR increasingly require that ML
predictions be accompanied by post-hoc explanations, even when raw data and trained models cannot be released.  Differential privacy~(DP) is the standard mitigation for the residual privacy risk of releasing these explanations.  We show that DP is not sufficient: an adversary observing only DP-perturbed GNN explanations can reconstruct hidden graph structure with high accuracy. Our attack, \textbf{\ours}, exploits the fact that the Gaussian DP mechanism is a single DDPM forward step at known noise level $\sigma(\eps)$, recasting reconstruction as \emph{reverse diffusion} conditioned on the corrupted signal, a principled Bayesian denoiser under known DP corruption.  We formalise a stratified adversary model parameterised by $(\calM, \hat\eps, \hat\delta, \mathcal{S}, \rho)$ that interpolates between oblivious and oracle attackers, and derive endpoint-matched two-sided bounds on reconstruction AUC. For practitioners, we provide regime-stratified guidance on explainer choice: on homophilic graphs, neighbourhood-aggregating explainers (GraphLIME, GNNExplainer) leak more structure than per-node gradient explainers under the same DP budget; on strongly heterophilic graphs the ordering reverses. We introduce \textbf{\ourf} as an auxiliary diagnostic sharing the same diffusion backbone to decompose leakage into explainer-induced and intrinsic graph-distribution components. Experiments across seven benchmarks, three DP mechanisms, and three GNN backbones show \ours\ achieves AUC above $0.7$ at $\eps=5$ on five of seven datasets, with the attack succeeding well within typically deployed privacy budgets. 
\end{abstract}

\section{Introduction}
\label{sec:intro}

Graph Neural Networks~(GNNs) are widely deployed in privacy-sensitive
domains: social network analysis, medical record linkage, and financial
fraud detection. In these deployments, two constraints typically apply
together. First, raw data and trained models cannot be released:
training graphs encode sensitive relationships, and access to the model
itself enables direct extraction attacks. Second, regulatory frameworks
such as the GDPR right to explanation\footnote{See e.g.\ Articles~13--15 and 22 of EU Regulation 2016/679, and analogous provisions in the EU AI Act and various US sectoral regulations.} increasingly
require that automated decisions be accompanied by a post-hoc
explanation. Practitioners are therefore obliged to release exactly
one artefact: the explanation. Post-hoc methods such as
Grad~\citep{simonyan2014deep},
GradInput~\citep{sundararajan2017axiomatic},
GNNExplainer~\citep{ying2019gnnexplainer}, and
GraphLIME~\citep{huang2022graphlime} produce per-node feature
attribution scores that satisfy this requirement.

\citet{olatunji2023private} show that these explanations can
reconstruct hidden graph connectivity with high accuracy when released
in the clear, a severe privacy risk when edges encode social
relationships or medical associations. The natural mitigation is
\emph{differential privacy}~(DP)~\citep{dwork2006calibrating}, which
perturbs explanation outputs before release. We show that DP at
deployed budgets is not sufficient: an adaptive attacker that models
the noise distribution can reconstruct graph structure with non-trivial
accuracy across all seven benchmarks we evaluate.

The risk extends beyond raw edge recovery: even partial reconstruction
of graph structure enables downstream inference attacks such as
community detection, link prediction, and membership inference,
amplifying privacy leakage in deployed systems. Understanding and
quantifying this \emph{structural leakage} under the only defence the
deployment setting allows is therefore critical for privacy-preserving
explainable AI.  

Three important gaps persist in the literature. \textbf{(i)}~Existing
reconstruction methods are heuristic and ignore the joint distribution
$p(\A, \Phi \mid \X)$, leaving open whether learned global priors
offer a quantifiable advantage over similarity-based attacks.
\textbf{(ii)}~No formal adversary model captures \emph{partial}
knowledge of the DP mechanism, yielding threat models that are either
oblivious to noise or assume oracle access to $(\eps,\delta)$, neither
of which corresponds to a realistic attacker.
\textbf{(iii)}~Reconstruction attacks have been studied almost
exclusively on homophilic citation networks, leaving the heterophilic
regime, and the relationship between explanation fidelity and edge
fidelity in that regime, unexplored.

\paragraph{Our contributions.}
\begin{enumerate}[leftmargin=*,itemsep=1pt,topsep=2pt]
  \item \textbf{DP-protected GNN explanations remain a viable attack
        surface.} Our main attack \ours\ recovers hidden adjacency from
        explanations released under standard $(\eps,\delta)$-DP across
        seven benchmark graphs, three DP mechanisms, and three GNN
        backbones. Reconstruction AUC stays substantially above chance
        for all $\eps\geq 0.5$ and approaches chance only at
        $\eps\leq 0.1$, a regime well below typical deployed budgets.
  \item \textbf{Stratified adversary model and two-sided AUC bounds.}
        We introduce a knowledge tuple
        $(\calM, \hat\eps, \hat\delta, \mathcal{S}, \rho)$ that
        interpolates between oblivious~(Type~I) and oracle~(Type~III)
        attackers, and derive endpoint-matched upper and lower bounds
        on the partition-weighted Type-II AUC $\overline{R}_{\mathrm{II}}$, a partition-reweighted proxy for the full-set AUC reported
        in \cref{sec:experiments} (the two coincide when classifier
        score distributions are partition-invariant; see \cref{rem:auc-proxy})
        (\cref{thm:advantage,thm:advantage_lower,cor:bracket}).
        The bracket has width $(R_{\mathrm{III}}-R_{\mathrm{I}})\cdot 2\rho(1-\rho)$,
        vanishes at $\rho\in\{0,1\}$, and \emph{bounds} the realistic
        partial-adaptive threat model corresponding to a deployed
        DP explanation API.
  \item \textbf{Diffusion as inverse denoising.}
        We identify the Gaussian DP mechanism with a forward diffusion
        step at noise level $\sigma(\eps)$, casting graph reconstruction
        as a \emph{posterior inference problem} under known corruption.
        Reverse diffusion thus acts as a principled Bayesian denoiser
        that captures global dependencies in $p(\A \mid \tilde{s})$,
        going beyond local similarity-based heuristics
        (\cref{sec:method}).
  \item \textbf{Regime-stratified guidance for the explainer-choice
        problem.} On homophilic graphs, neighbourhood-aggregating
        explainers leak more structure than per-node gradient
        explainers under the same DP budget~(\cref{lem:fidelity_gap}).
        On strongly heterophilic graphs (e.g., IMDB, $h=0.19$) the
        ordering reverses; on milder heterophily the leakage gap
        narrows, because label-aligned explanations decorrelate from
        edge structure (\cref{prop:hetero}, \cref{app:fidelity}). This gives
        practitioners actionable guidance for the only privacy lever
        they control: which explainer to deploy.
  \item \textbf{\ourf\ as a diagnostic for intrinsic vs.\
        explainer-induced leakage~(\cref{app:privf_details}).}
        \ourf\ is an auxiliary attack on DP-perturbed raw features
        that shares the same diffusion architecture as \ours\ and
        serves purely as a leakage-decomposition tool: the
        \ours~$-$~\ourf\ gap separates leakage attributable to the
        explanation pipeline from leakage already present in the
        underlying graph distribution. Full methodology, theory, and
        results are in \cref{app:privf_details}.
\end{enumerate}

\section{Problem Formulation}
\label{sec:problem}

\paragraph{Setup.}
Let $\G = (\mathcal{V}, \edgeset, \X)$ be a graph with
$\mathcal{V} = \{1,\ldots,n\}$, adjacency matrix
$\A \in \{0,1\}^{n\times n}$, and node features $\X \in \R^{n \times d}$.
A GNN $\ftheta$ is trained on $\G$ for $C$-class node classification.
For node $v$, an explainer yields
$\phi_v = \mathrm{Explainer}(\ftheta, v, \X, \A) \in \R^d$.

\paragraph{DP mechanisms.}
We study three standard mechanisms applied to $\phi_v$ (explanation)
or $x_v$ (raw feature):
\emph{Gaussian}:
$\tilde\phi_v = \phi_v + \eta_v$,
$\eta_v \sim \calN(0, \sigma_G^2 I_d)$,
$\sigma_G = \sqrt{2\ln(1.25/\delta)}\,\Delta_f/\eps$,
satisfying $(\eps,\delta)$-DP~\citep{dwork2006calibrating};
\emph{Laplace}: $\eta_v \sim \mathrm{Lap}(0,\Delta_f/\eps)^d$, pure $\eps$-DP;
\emph{R\'enyi~(RDP)}: $\sigma_R = \sqrt{\alpha\Delta_f^2/(2\eps_{\mathrm{rdp}})}$
for order $\alpha$, convertible to $(\eps,\delta)$-DP~\citep{mironov2017renyi}.

\paragraph{Reconstruction tasks.}
\textbf{\ours}~(explanation attack, primary): an adversary observes
$\Phitilde = \{\tilde\phi_v\}_{v\in\mathcal{S}}$ and recovers $\A$.
This is the deployment-relevant attack: explanations are precisely
what regulatory frameworks force the system to release.
\textbf{\ourf}~(feature attack, diagnostic): the adversary observes
$\tilde\X_\mathcal{S} = \{x_v + \eta_v\}_{v\in\mathcal{S}}$ with no
access to any GNN model or explanation, and recovers $\A$. \ourf\ is
not a competing attack in any deployment scenario; it is purely a
diagnostic that uses the same diffusion architecture as \ours\ with
DP-perturbed features as input. The \ours~$-$~\ourf\ gap decomposes
leakage into explainer-induced and graph-distribution components.
Full \ourf\ methodology, theory, and results appear in
\cref{app:privf_details}.
Performance is measured by Average Precision~(AP) and AUC-ROC.

\section{Stratified Adversary Model}
\label{sec:threat}

A fundamental gap in prior work is the treatment of the adversary as either
naive~(ignoring DP) or omniscient~(knowing $\eps, \delta$ exactly).
 
\begin{definition}[Adversary Knowledge Tuple]
  \label{def:knowledge}
  An adversary $\calA$ is parameterised by
  $\calK = (\calM, \hat\eps, \hat\delta, \mathcal{S}, \rho)$ where
  $\calM\!\in\!\{\text{Gaussian},\text{Laplace},\text{RDP},\text{unknown}\}$,
  $\hat\eps,\hat\delta$ are estimated privacy parameters with relative
  error $\kappa = |\hat\eps - \eps|/\eps$, $\mathcal{S}\subseteq\mathcal{V}$
  is the observed subset, and $\rho = |\mathcal{S}|/n$.
\end{definition} 

\begin{definition}[Adversary Types]
  \label{def:types}
  \textbf{Type~I~(Oblivious)}: ignores DP noise; equivalent to
  \citet{olatunji2023private}.
  \textbf{Type~II~(Partial-Adaptive)}: knows mechanism type and
  approximate budget ($\kappa \leq \kappa_0$), observes fraction
  $\rho\in(0,1)$ of nodes.~\emph{This is the adversary modelled by
  \ours/\ourf.}
  \textbf{Type~III~(Oracle)}: exact knowledge, full observation.
\end{definition}

We write $R_{\mathrm{I}} \leq R_{\mathrm{half}} \leq R_{\mathrm{III}}$ for the AUC of the oblivious, half-edge, and oracle classifiers respectively (data-processing inequality; see \cref{app:proof_advantage} for precise definitions and \cref{lem:partition} for the edge-partition probabilities used below).

\begin{theorem}[Partial-Adaptive Attacker: Upper Bound]
\label{thm:advantage}
Let $\overline{R}_{\mathrm{II}}$ denote the partition-weighted AUC of the Type-II classifier (see \cref{app:proof_advantage} for the precise definition). Under the standing assumptions,
\begin{equation}
\label{eq:thm1_upper}
\overline{R}_{\mathrm{II}} \;\leq\;
R_{\mathrm{III}} \cdot p_{SS} \;+\; R_{\mathrm{half}} \cdot p_{SU} \;+\; R_{\mathrm{I}} \cdot p_{UU}.
\end{equation}
In the limit $n\to\infty$, applying $R_{\mathrm{half}}\leq R_{\mathrm{III}}$:
\begin{equation}
\label{eq:thm1_upper_clean}
\overline{R}_{\mathrm{II}} \;\leq\;
R_{\mathrm{III}}\!\left(1-(1-\rho)^2\right) \;+\; R_{\mathrm{I}}(1-\rho)^2.
\end{equation}
\end{theorem}

\begin{theorem}[Partial-Adaptive Attacker: Lower Bound]
\label{thm:advantage_lower}
As $\kappa\to 0$ (or when $\hat\sigma$ is consistent for $\sigma$),
\begin{equation}
\label{eq:thm1_lower}
\overline{R}_{\mathrm{II}} \;\geq\; R_{\mathrm{III}} \cdot p_{SS} \;+\; R_{\mathrm{I}}(p_{SU}+p_{UU}).
\end{equation}
In the limit $n\to\infty$: $\overline{R}_{\mathrm{II}} \geq R_{\mathrm{I}} + (R_{\mathrm{III}}-R_{\mathrm{I}})\rho^2$.
For finite estimation error $\kappa$, replace $R_{\mathrm{III}}$ with $R_{\mathrm{II}}^{SS}(\kappa)\leq R_{\mathrm{III}}$, where $R_{\mathrm{III}} - R_{\mathrm{II}}^{SS}(\kappa)\leq C_{\mathrm{deg}}\cdot\kappa$ for a problem-dependent constant $C_{\mathrm{deg}}$.
\end{theorem}

\begin{corollary}[Endpoint-Matched Bracket]
\label{cor:bracket}
Combining \cref{thm:advantage,thm:advantage_lower} in the limit $n\to\infty$:
\begin{equation}
\label{eq:bracket}
\underbrace{R_{\mathrm{I}} + (R_{\mathrm{III}}-R_{\mathrm{I}})\rho^2}_{\text{lower bound}} \;\leq\; \overline{R}_{\mathrm{II}} \;\leq\; \underbrace{R_{\mathrm{I}} + (R_{\mathrm{III}}-R_{\mathrm{I}})(1-(1-\rho)^2)}_{\text{upper bound}}.
\end{equation}
The bracket is tight at $\rho\in\{0,1\}$ and has maximal width $(R_{\mathrm{III}}-R_{\mathrm{I}})/2$ at $\rho=\tfrac{1}{2}$. As $\rho\!\to\!0$, $\overline{R}_{\mathrm{II}}\!\to\!R_{\mathrm{I}}$ (oblivious); as $\rho\!\to\!1$, $\overline{R}_{\mathrm{II}}\!\to\!R_{\mathrm{III}}$ (oracle). The bracket is matched at the endpoints, not uniformly; the $\rho=0.5$ gap quantifies the irreducible uncertainty introduced by partial observation.
\end{corollary}

Proofs and elaborating remarks (including the AUC-proxy interpretation and
the empirical calibration of the bracket gap) are in \cref{app:proofs}. Setting $\kappa_0=0.3$ and $\rho=0.5$ in our
experiments models a realistic deployment where the attacker queries half
the nodes and estimates the noise scale from the observed variance of
$\Phitilde$.

\section{Method: Adaptive Diffusion Reconstruction}
\label{sec:method}

\subsection{Architecture Overview}

\begin{figure*}[t]
\centering
\includegraphics[width=\linewidth]{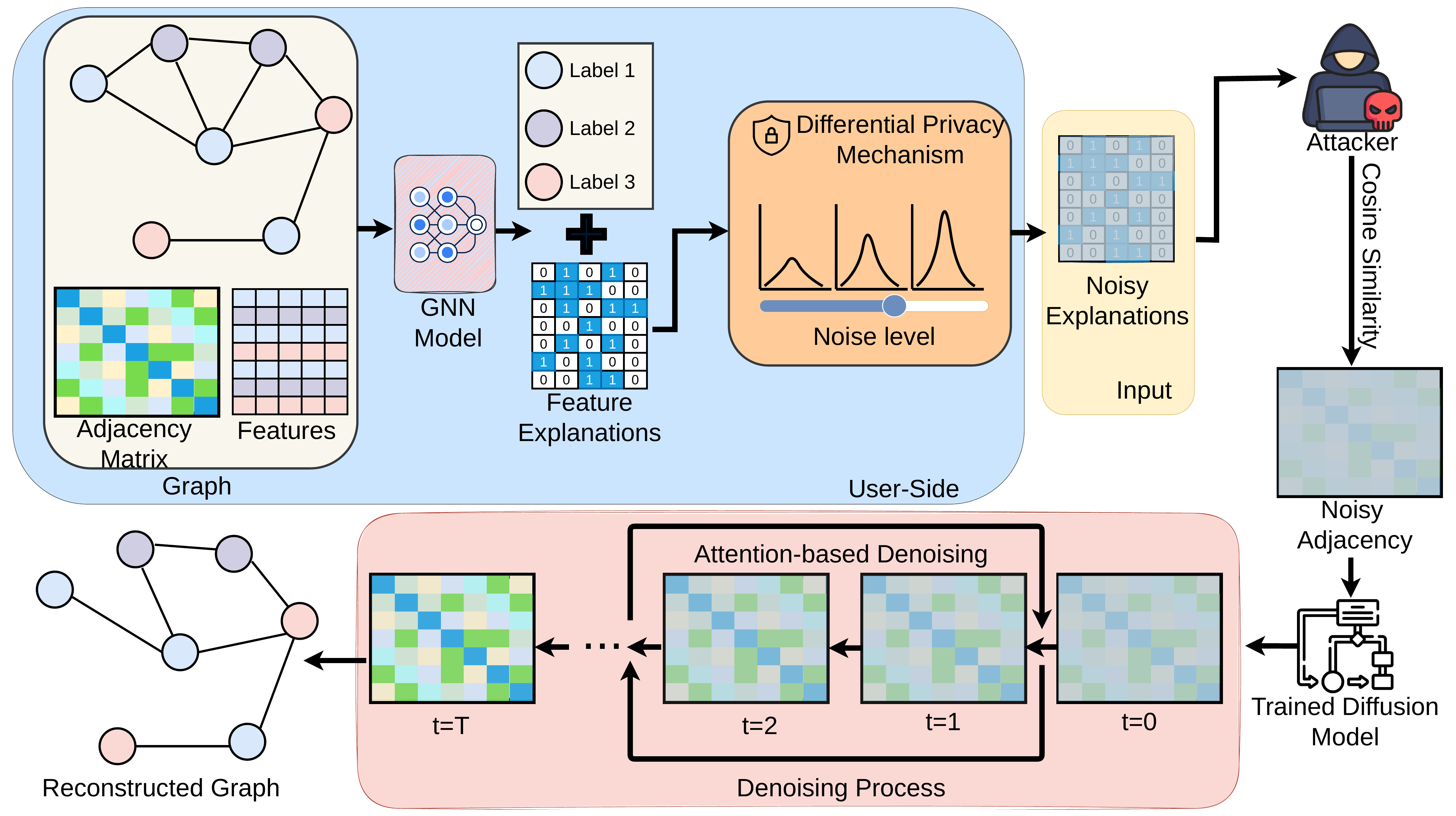}
\caption{Overview of the \ours\ framework for reconstructing graph structure from differentially private explanations. 
\textbf{Top (user-side):} the input graph is processed by a GNN to produce node-level explanations $\phi_v$, which are perturbed by a differential privacy (DP) mechanism to yield noisy explanations $\tilde{\phi}_v$ at a given noise level. An adaptive attacker observes these noisy explanations and constructs an initial noisy adjacency estimate. 
\textbf{Bottom (attacker-side):} a trained diffusion model performs iterative denoising from timestep $t=T$ to $t=0$, refining a latent adjacency representation $z_t$ into a reconstructed graph. At each step, the denoiser $\epsilon_\theta$ (an attention-based graph transformer) conditions on the noisy explanations via cross-attention and incorporates timestep information, enabling posterior inference under the DP noise model.}
\label{fig:privx}
\end{figure*}

\cref{fig:privx} illustrates the \ours\ pipeline. The central architectural principle is that \emph{reconstruction is driven entirely by the conditioning signal}, while the generative backbone remains unchanged. \ours\ conditions on noisy explanations $\tilde{\phi}_v$. An auxiliary variant, \ourf\ (see \cref{app:privf_details}), conditions on noisy features $\tilde{x}_v$ and shares an identical diffusion backbone; it is used purely as a leakage-decomposition diagnostic and does not correspond to a real deployment scenario.

\subsection{Diffusion as a Denoiser for DP Signals}

\paragraph{DP noise as forward diffusion.}
In a DDPM~\citep{ho2020denoising} with cumulative noise product
$\bar\alpha_t$, the forward process satisfies
$q(\mathbf{z}_t|\mathbf{z}_0) = \calN(\sqrt{\bar\alpha_t}\mathbf{z}_0,
(1-\bar\alpha_t)I)$.
The Gaussian DP perturbation $\tilde\phi_v = \phi_v + \calN(0,\sigma^2 I)$
corresponds to this process at an effective timestep
\begin{equation}
  t^\star = \argmin_t \,\abs{(1-\bar\alpha_t) - \sigma^2(\eps)}.
\end{equation} 

\paragraph{Reverse diffusion as posterior inference.}
Given noisy observations $\tilde{s}$ (either $\Phitilde$ or
$\tilde\X_\mathcal{S}$), reconstruction corresponds to solving the
posterior inference problem
\begin{equation}
  p(A,\Phi\,|\,\Phitilde) \;\propto\;
  p(\Phitilde|\Phi)\,p(\Phi|A,\X)\,p(A|\X).
\end{equation}
Unlike generative diffusion models that learn $p(A)$, our formulation
conditions on the observed corrupted signal and performs inference in
$p(A \mid \tilde{s})$. Reverse diffusion thus acts as a \emph{Bayesian
denoiser} under a known corruption model, progressively removing DP
noise while preserving globally consistent graph structure.

The graph-structured analogy to DDRM~\citep{kawar2022denoising} and discrete-structure considerations are in \cref{app:implementation}.

\subsection{Denoising Network and Training}

The denoising network $\eps_\theta$ is a graph transformer trained via:
\begin{equation}
  \mathcal{L} = \E_{\A,\,\tilde{s},\,t}\!\left[
    \norm{\eps - \eps_\theta(\mathbf{z}_t,\,\tilde{s},\,t)}^2
  \right],
\end{equation}
where $\tilde{s}=\Phitilde$ for \ours\ (or $\tilde{s}=\tilde\X$ for the \ourf\ diagnostic in \cref{app:privf_details}).
At inference, $t^\star$ is estimated from the observed signal variance:
$\hat\sigma^2 = \mathrm{Var}(\tilde{s}) - \widehat{\mathrm{Var}}(s)$.
The full procedure is given in \cref{alg:privx}.

\begin{algorithm}[t]
\caption{\ours\ (and \ourf\ in \cref{app:privf_details}): Adaptive Diffusion Graph Reconstruction}
\label{alg:privx}
\begin{algorithmic}[1]
\Require Noisy signal $\tilde{s}$ (explanations; or features for \ourf\ in \cref{app:privf_details}), denoiser $\eps_\theta$, steps $T$,
         schedule $\{\alpha_t,\beta_t,\sigma_t\}$
\Ensure Reconstructed adjacency $\Ahat$
\State $\hat\sigma^2 \gets \sigma^2(\hat\eps)$
       \Comment{from DP calibration formula for $\hat{\calM}$; see \cref{sec:threat}}
\State $t^\star \gets \argmin_t |(1-\bar\alpha_t) - \hat\sigma^2|$
\State $\mathbf{z}_{t^\star} \sim \calN(0,I)$
\For{$t = t^\star, \ldots, 1$}
  \State $\hat\eps \gets \eps_\theta(\mathbf{z}_t, \tilde{s}, t)$
         \Comment{cross-attend to noisy signal}
  \State $\mathbf{z}_{t-1} \gets \frac{1}{\sqrt{\alpha_t}}\!\left(\mathbf{z}_t
    - \frac{\beta_t}{\sqrt{1-\bar\alpha_t}}\hat\eps\right) + \sigma_t\eta$,
    $\;\eta \sim \calN(0,I)$
\EndFor
\State $\Ahat \gets \sigma(\mathbf{z}_0)$ \Comment{sigmoid decode + threshold at 0.5}
\State \Return $\Ahat$
\end{algorithmic}
\end{algorithm}

Full architecture details and complexity analysis are in \cref{app:implementation}.

\section{Theoretical Analysis}
\label{sec:theory}

\begin{definition}[Graph Homophily]
  $h = \Pr_{(u,v)\sim\edgeset}[\mathrm{label}(u)=\mathrm{label}(v)]$.
\end{definition}

\begin{definition}[Edge-fidelity]
\label{def:edge_fid}
For an explainer producing $\phi_v\in\R^d$, the \emph{edge-fidelity} is
\[
  \gamma_E \;:=\; \Pr_{(i,j)\sim\edgeset}\!\left[
    \langle\hat\phi_i,\hat\phi_j\rangle >
    \langle\hat\phi_i,\hat\phi_k\rangle \text{ for } k\sim
    \mathrm{Unif}(\V\setminus\calN(i))\right],
\]
where $\hat\phi=\phi/\norm{\phi}_2$. This measures how well the
explanation distinguishes connected from disconnected node pairs, the
signal an inner-product reconstruction attacker exploits.
\end{definition}

\begin{definition}[Label-fidelity]
\label{def:label_fid}
$\gamma_L = \Pr_v[\arg\max_c \tilde y_v^{(c)} = y_v]$ where $\tilde y$
is the prediction restricted to top-$k$ features by $|\phi|$. This is
the standard interpretability notion used by GNNExplainer and
GraphLIME. The two fidelities are aligned on homophilic graphs but
can decouple substantially on heterophilic graphs
(\cref{app:fidelity}).
\end{definition}

The bound in \cref{thm:reconstruction} is stated in terms of
$\gamma_E$, the attack-relevant fidelity; see \cref{app:proofs} for
the sign-fidelity $\gamma_S$ used in the proof of \cref{lem:fidelity_gap}.

\begin{theorem}[Edge Reconstruction TPR under Gaussian DP]
\label{thm:reconstruction}
Let $\bar\phi := \E[\norm{\phi_v}_2^2]$. Under homophily $h$,
edge-fidelity $\gamma_E$, and $(\eps,\delta)$-DP with Gaussian noise
scale $\sigma(\eps)$, the inner-product threshold attacker satisfies
\begin{equation}
\Pr[\Ahat_{ij}=1\mid A_{ij}=1]
\;\geq\;
1 - \frac{4\,\mathrm{Var}(\phi_i^\top\phi_j\mid A_{ij}=1)
+ 8\sigma^2\bar\phi + 4\sigma^4 d}{(h\gamma_E\bar\phi)^2}.
\end{equation}
The crossover scale is $\sigma_c^2 = h\gamma_E\bar\phi$; for Cora ($h=0.81$, $\hat\gamma_E\approx 0.7$) this gives $\eps_c\approx 0.5$.
Regime and deployment analysis in \cref{app:proofs}.
\end{theorem}

\begin{proposition}[Privacy--Utility Tradeoff]
  \label{prop:tradeoff}
  $R(\eps) = O(1/(1+\sigma^2(\eps)))$, with
  $R\to h\gamma_E$ as $\eps\to\infty$ and $R\to 0$ as $\eps\to 0$.
\end{proposition}

\subsection{Fidelity Gap Between Explanation Types (Homophilic Regime)}
\label{sec:fidelity_gap}

\begin{assumption}[Gradient Variance]
  \label{ass:grad_var}
  $\norm{\nabla_x\ell}_2\leq S_f$;
  $\mathrm{Var}(\phi^\mathrm{Grad}) = \Omega(d^{-1/2}S_f^2)$.
\end{assumption}

\begin{lemma}[Sign-Fidelity Gap, Homophilic Regime]
\label{lem:fidelity_gap}
Fix the homophilic regime $h\geq h_0 > 1/2$. Under \cref{ass:grad_var}
and Gaussian DP at scale $\sigma$, the post-DP per-coordinate
sign-fidelities satisfy
\begin{equation}
\mathbb{E}\!\left[\gamma_S^{\mathrm{LIME}} - \gamma_S^{\mathrm{Grad}}\right]
\;\geq\;
\Omega\!\left(\frac{\mathbb{E}[\sqrt{|N(v)|}] - 1}{\sigma \sqrt{d}}\right).
\end{equation}
The advantage stems from neighbourhood aggregation reducing
per-coordinate noise variance by a factor $|\calN(v)|$. In the
heterophilic regime ($h < 1/2$) the lemma does not apply: aggregation
aligns explanations with class boundaries which on heterophilic graphs
decorrelate from edge structure, and the ordering can reverse
(\cref{prop:hetero}, \cref{app:fidelity}). Expectation is over nodes
$v\sim\mathrm{Unif}(\V)$ with the graph structure held fixed.
\end{lemma}

\Cref{cor:edge_fidelity_gap} (\cref{app:proofs}) extends this to an edge-fidelity gap of the same order under a homophilic alignment condition; a regime-dependent remark on the heterophilic inversion is also given there.

\subsection{Heterophilic Reconstruction}
\label{sec:hetero_theory}

For heterophilic graphs ($h\ll 1$), the lower bound in
\cref{thm:reconstruction} degrades. However, connected nodes in
heterophilic graphs share \emph{dissimilar} features, creating a
distinguishable anti-correlation signal.

\begin{proposition}[Heterophilic Reconstruction Lower Bound]
\label{prop:hetero}
Consider a heterophilic graph with heterophily level $1-h$ and
feature--structure anti-correlation $\rho_{XA}^{-} > 0$.
Assume that inter-class feature separation exceeds intra-class variance,
so that for disconnected node pairs the expected inner product is
bounded away from that of true edges. Then
\begin{equation}
\Pr[\Ahat_{ij}=1\mid A_{ij}=1]
\;\geq\; (1-h)\cdot|\rho_{XA}^{-}| - O(\sigma^2).
\end{equation}
\end{proposition}
\noindent
The separation assumption holds empirically on IMDB and Amazon-Ratings; see \cref{app:proofs} for the false-positive analysis.

\subsection{Subgraph Sampling Bias for Large Graphs}
\label{prop:sampling_bias}
For power-law graphs such as ogbn-arxiv, boundary edges that lie
outside the sampled $k$-hop ego-net introduce a systematic bias that
limits reconstruction quality regardless of model capacity; the
precise statement and proof are in \cref{app:proof_sampling_bias}.

\section{Experiments}
\label{sec:experiments}

\subsection{Setup}
We evaluate on seven benchmarks spanning homophilic (Cora, CiteSeer, PubMed), large-scale (ogbn-arxiv), and heterophilic (Chameleon, IMDB, Amazon-ratings) graphs. We consider four explainers (Grad, GradInput, GNNExplainer, GraphLIME) and compare \textbf{\ours} against two baselines: ExplainSim and GSE. Feature-based methods (FeatureSim, SLAPS) and \ourf\ results appear in \cref{app:privf_details}. Experiments use GCN, GIN, and GraphSAGE backbones under DP budgets $\eps\in\{0.1,0.5,1,2,5,8,16\}$ with $\delta=10^{-5}$ and Gaussian, Laplace, and R\'enyi mechanisms. Models are trained with standard settings (3-layer GNNs, hidden dim 256, dropout 0.5, Adam), while the diffusion model uses a 4-layer graph transformer (hidden dim 256, 4 heads, cosine schedule, $T=200$). We evaluate across window sizes $k\in\{32,64,128\}$ and multiple train/test splits. All experiments are conducted on 4$\times$ NVIDIA RTX 6000 Ada GPUs. Full dataset statistics are provided in Appendix~\cref{app:datasets}, and additional implementation details in Appendix~\cref{app:implementation}.

\subsection{Result Analysis}
\label{sec:main_results}

\cref{tab:main_eps5} reports reconstruction AP and AUC at $\eps=5.0$
(Gaussian DP, $w=32$) for both GCN and GraphSAGE backbones across all
seven datasets, comparing ExplainSim, GSE, and \ours\ for each
explainer. Bold entries indicate the best within each explainer group
(ExplainSim/GSE vs.\ \ours); OOM denotes out-of-memory or unavailable
results. Feature-based baselines (FeatureSim, SLAPS) and \ourf\ results
appear in \cref{tab:feature_methods} (\cref{app:privf_details}).

\begin{table*}[t]
  \centering
  \caption{Graph reconstruction performance at $\varepsilon=5$ (Gaussian DP, $w=32$; GCN and GraphSAGE backbones). \textbf{Bold}: best within each method group (ExplainSim/GSE vs.\ \textsc{PrivX}). Results are averaged over 5 independent runs, with standard deviation in the range $\pm 0.02$ to $\pm 0.04$. OOM: out-of-memory. Feature-based methods (FeatureSim, SLAPS, \textsc{PrivF}) appear in \cref{tab:feature_methods}.}
  \label{tab:main_eps5}
  \scriptsize
  \setlength{\tabcolsep}{2.8pt}
  \begin{tabular}{llcccccccccccccc}
  \toprule
  & & \multicolumn{8}{c}{Homophilic} & \multicolumn{6}{c}{Heterophilic} \\
  \cmidrule(lr){3-10} \cmidrule(lr){11-16}
  \textbf{Explainer} & \textbf{Method} & \multicolumn{2}{c}{\textbf{Cora}} & \multicolumn{2}{c}{\textbf{CiteSeer}} & \multicolumn{2}{c}{\textbf{PubMed}} & \multicolumn{2}{c}{\textbf{ogbn-arxiv}} & \multicolumn{2}{c}{\textbf{Chameleon}} & \multicolumn{2}{c}{\textbf{IMDB}} & \multicolumn{2}{c}{\textbf{Amz-R}} \\
  & & AUC & AP & AUC & AP & AUC & AP & AUC & AP & AUC & AP & AUC & AP & AUC & AP \\
  \midrule
  \multirow{4}{*}{\textsc{Grad}} & ExplainSim & 0.534 & 0.546 & 0.510 & 0.531 & 0.522 & 0.536 & 0.477 & 0.484 & 0.503 & 0.506 & 0.524 & 0.523 & 0.493 & 0.503 \\
   & GSE & 0.480 & 0.497 & 0.542 & 0.610 & 0.504 & 0.510 & OOM & OOM & 0.534 & 0.528 & 0.475 & 0.497 & 0.485 & 0.485 \\
   & \textbf{\textsc{PrivX}-GCN (Ours)} & \textbf{0.736} & \textbf{0.764} & \textbf{0.801} & \textbf{0.824} & 0.573 & 0.605 & 0.558 & 0.565 & 0.648 & 0.627 & \textbf{0.819} & \textbf{0.803} & 0.570 & 0.581 \\
   & \textbf{\textsc{PrivX}-SAGE (Ours)} & 0.670 & 0.669 & 0.666 & 0.678 & \textbf{0.709} & \textbf{0.720} & \textbf{0.652} & \textbf{0.657} & \textbf{0.741} & \textbf{0.728} & 0.755 & 0.719 & \textbf{0.622} & \textbf{0.630} \\
  \midrule
  \multirow{4}{*}{\shortstack{\textsc{Grad-}\\\textsc{Input}}} & ExplainSim & 0.521 & 0.522 & 0.547 & 0.553 & 0.475 & 0.494 & 0.514 & 0.510 & 0.514 & 0.531 & 0.517 & 0.516 & 0.503 & 0.504 \\
   & GSE & 0.544 & 0.513 & 0.581 & 0.638 & 0.486 & 0.493 & OOM & OOM & 0.533 & 0.521 & 0.487 & 0.495 & 0.506 & 0.505 \\
   & \textbf{\textsc{PrivX}-GCN (Ours)} & \textbf{0.706} & \textbf{0.745} & \textbf{0.793} & \textbf{0.808} & 0.686 & 0.700 & 0.555 & 0.565 & 0.544 & 0.544 & \textbf{0.734} & \textbf{0.720} & 0.512 & 0.521 \\
   & \textbf{\textsc{PrivX}-SAGE (Ours)} & 0.652 & 0.658 & 0.665 & 0.681 & \textbf{0.729} & \textbf{0.745} & \textbf{0.650} & \textbf{0.656} & \textbf{0.674} & \textbf{0.676} & 0.715 & 0.677 & \textbf{0.596} & \textbf{0.616} \\
  \midrule
  \multirow{4}{*}{\shortstack{\textsc{Graph-}\\\textsc{LIME}}} & ExplainSim & 0.520 & 0.558 & 0.516 & 0.527 & 0.465 & 0.494 & 0.500 & 0.499 & 0.518 & 0.523 & 0.532 & 0.532 & 0.511 & 0.517 \\
   & GSE & 0.506 & 0.518 & 0.541 & 0.608 & 0.497 & 0.502 & OOM & OOM & 0.519 & 0.544 & 0.520 & 0.527 & 0.520 & 0.525 \\
   & \textbf{\textsc{PrivX}-GCN (Ours)} & \textbf{0.717} & \textbf{0.751} & 0.665 & \textbf{0.701} & 0.560 & 0.586 & 0.516 & 0.525 & 0.515 & 0.548 & 0.472 & 0.512 & 0.512 & 0.522 \\
   & \textbf{\textsc{PrivX}-SAGE (Ours)} & 0.705 & 0.710 & \textbf{0.683} & 0.700 & \textbf{0.745} & \textbf{0.756} & \textbf{0.653} & \textbf{0.665} & \textbf{0.594} & \textbf{0.611} & \textbf{0.732} & \textbf{0.698} & \textbf{0.592} & \textbf{0.609} \\
  \midrule
  \multirow{4}{*}{\shortstack{\textsc{GNNEx-}\\\textsc{plainer}}} & ExplainSim & 0.473 & 0.513 & 0.510 & 0.500 & 0.489 & 0.508 & 0.502 & 0.504 & 0.489 & 0.527 & 0.526 & 0.517 & 0.503 & 0.508 \\
   & GSE & 0.496 & 0.512 & 0.553 & 0.629 & 0.497 & 0.502 & OOM & OOM & 0.536 & 0.545 & 0.511 & 0.506 & 0.529 & 0.518 \\
   & \textbf{\textsc{PrivX}-GCN (Ours)} & \textbf{0.836} & \textbf{0.860} & \textbf{0.897} & \textbf{0.909} & 0.701 & 0.715 & 0.515 & 0.524 & 0.586 & 0.589 & 0.483 & 0.515 & 0.512 & 0.524 \\
   & \textbf{\textsc{PrivX}-SAGE (Ours)} & 0.774 & 0.771 & 0.685 & 0.700 & \textbf{0.710} & \textbf{0.733} & \textbf{0.645} & \textbf{0.647} & \textbf{0.800} & \textbf{0.785} & \textbf{0.751} & \textbf{0.738} & \textbf{0.601} & \textbf{0.615} \\
  \bottomrule
  \end{tabular}
\end{table*}

\textbf{(1) GraphSAGE consistently improves over GCN.}
\ours-SAGE outperforms \ours-GCN across all datasets. The gains are
largest on heterophilic datasets: on Chameleon, \ours-SAGE (GNNExplainer)
achieves AP~$=0.785$ vs.\ GCN's~$0.589$ ($+19.6$\,pp).
GraphSAGE's mean-pooling aggregation yields smoother, more discriminative
node representations that better preserve structural signal after DP noise.
\ourf\ results showing the same backbone effect are in
\cref{tab:feature_methods} (\cref{app:privf_details}).

\textbf{(2) Diffusion outperforms all heuristic baselines by large margins.}
\ours\ substantially improves over ExplainSim and GSE: on Cora, \ours-SAGE with GNNExplainer achieves AP~$=0.775$ vs.\ ExplainSim~$=0.513$ ($+51\%$) and GSE~$=0.497$ ($+56\%$). Feature-based baseline comparisons are in \cref{tab:feature_methods} (\cref{app:privf_details}).

\textbf{(3) On homophilic graphs, surrogate explainers leak more.}
GNNExplainer and GraphLIME \ours\ variants outperform Grad and GradInput
on every homophilic dataset, confirming \cref{lem:fidelity_gap}.
On Cora with GraphSAGE, \ours-SAGE GNNExplainer AP~$=0.775$ vs.\
\ours-SAGE Grad~$=0.639$ ($+13.6$\,pp). The gap narrows on heterophilic
datasets (Chameleon, IMDB, Amazon-ratings); on IMDB the ordering
reverses, with Grad-\ours-GCN (AP~$=0.803$) substantially exceeding
GNNExplainer-\ours-GCN (AP~$=0.515$). This is the empirical signature
of \cref{prop:hetero}: when the GNN classifies via dissimilar-feature
neighbours, surrogate explanations align with label boundaries that
decorrelate from edge structure, while gradient explanations expose
the anti-correlation signal that drives heterophilic reconstruction.
\cref{app:fidelity} develops this finding in detail.

\textbf{(4) DP alone is not sufficient at deployed budgets.}
\ours\ achieves AP $\geq 0.6$ in 42 of 56 (explainer, dataset, backbone) configurations at $\eps=5$, and AP $\geq 0.7$ in 24 of 56.
AP remains above chance at $\eps=0.5$ on both Cora and IMDB, collapsing only at $\eps=0.1$, consistent with $\eps_c\approx 0.5$ (\cref{thm:reconstruction}).

\textbf{(5) Decomposing total leakage with \ourf\ (\cref{app:privf_details}).}
On homophilic graphs (Cora, CiteSeer), \ours\ modestly exceeds \ourf\ (surrogate explanations add structural signal);
on IMDB, \ours\ trails \ourf\ (most recoverable structure is in the graph distribution).
Full results in \cref{tab:feature_methods}; DP on explanations protects less total leakage than a naive view suggests.

\subsection{Privacy Budget Analysis}

\begin{figure}[t]
  \centering
  \includegraphics[width=\linewidth]{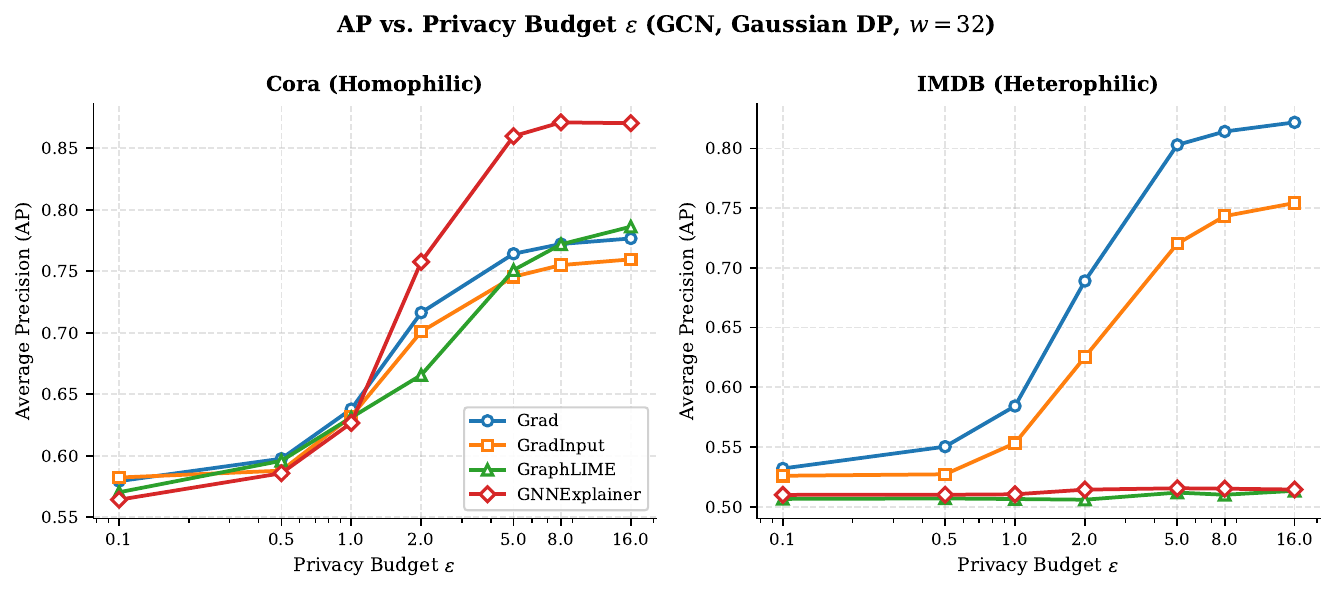}
  \caption{Reconstruction AP vs.\ privacy budget $\eps$ for four explainers
  on Cora~(homophilic) and IMDB~(heterophilic) (GraphSAGE backbone,
  Gaussian DP, $w=32$). AP increases monotonically with $\eps$ in both
  settings, consistent with \cref{prop:tradeoff}. On Cora, GNNExplainer
  dominates gradient-based explainers, consistent with the homophilic
  edge-fidelity gap~(\cref{lem:fidelity_gap}); on heterophilic IMDB the
  ordering narrows and partially reverses.}
  \label{fig:privacy_budget}
\end{figure}

\cref{fig:privacy_budget} shows AP increases monotonically with $\eps$ on both Cora and IMDB. On Cora, GNNExplainer dominates (highest SNR after aggregation, confirming \cref{lem:fidelity_gap}); on IMDB the explainer gap narrows because the anti-correlation signal is uniformly accessible across explainer families. Even at $\eps=0.1$, AP remains above $0.52$ on both datasets; typical deployed budgets ($\eps\in[1,8]$) sit well above the crossover $\eps_c\approx 0.5$ predicted by \cref{thm:reconstruction}.

\subsection{Adaptive Attacker Analysis}

\cref{tab:adaptive_cora_gnnexplainer} and
\cref{fig:adaptive} (\cref{app:adaptive}) present the adaptive attacker ablation on Cora
(GNNExplainer, GraphSAGE backbone, $\eps=5.0$). Four patterns confirm
\cref{thm:advantage,thm:advantage_lower}:
\textbf{(i) Smooth degradation with $\kappa$:} AP drops only $3.5$\,pp from $\kappa=0$ to $\kappa=1.0$, confirming $C_{\mathrm{deg}}\lesssim 0.035$.
\textbf{(ii) Monotonic gain with $\rho$:} AP rises $+10.2$\,pp from $\rho=0.25$ to $1.0$, consistent with the $\rho^2$ lower bound in \cref{cor:bracket}.
\textbf{(iii--iv)} The $\rho$-gain is stronger on Cora (homophilic) than on Amazon-ratings where it is non-monotonic; GraphSAGE explanations yield higher AP at every $(\rho,\kappa)$ point than GCN. Full tables in \cref{app:adaptive_tables}.

\subsection{Ablation Studies}

\begin{figure}[t]
\centering
\begin{subfigure}[b]{0.49\linewidth}
  \includegraphics[width=\linewidth]{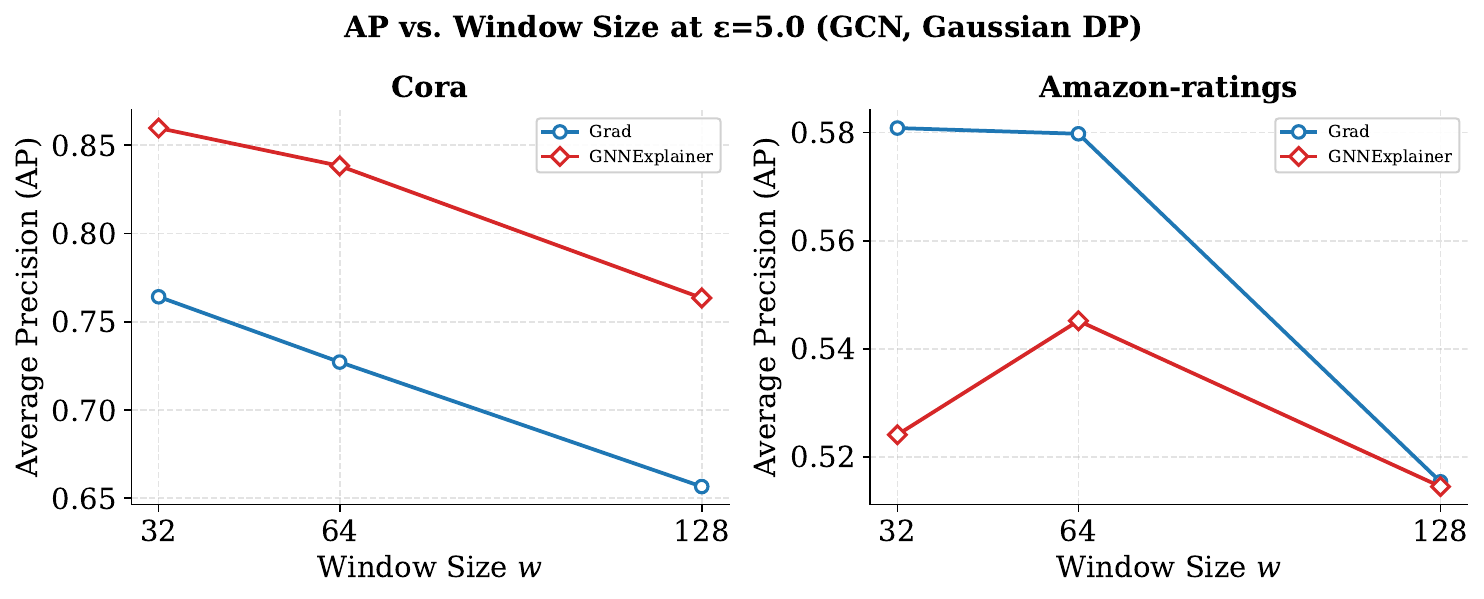}
  \caption{Window size vs.\ AP}
  \label{fig:window}
\end{subfigure}
\hfill
\begin{subfigure}[b]{0.49\linewidth}
  \includegraphics[width=\linewidth]{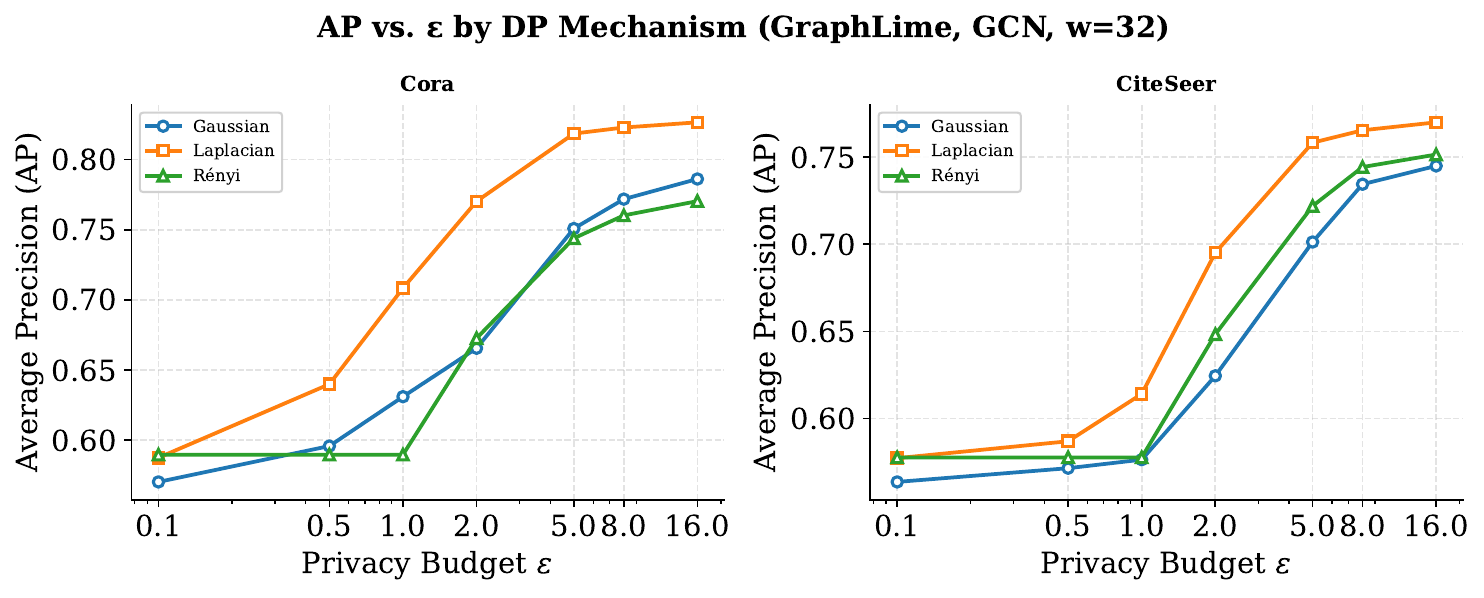}
  \caption{DP mechanism vs.\ AP}
  \label{fig:mechanism}
\end{subfigure}
\caption{\textbf{Left:} AP vs.\ subgraph window size~$k$ at $\eps=5.0$.
Performance saturates near $k=128$ on Cora and earlier on Amazon-ratings,
consistent with \cref{prop:sampling_bias}. \textbf{Right:} AP vs.\ $\eps$
across three DP mechanisms~(GraphLIME, GCN, $w=32$). Laplace achieves
slightly higher AP for the same nominal $\eps$ due to heavier tails
preserving more high-frequency structure signal.}
\label{fig:ablations}
\end{figure}

\paragraph{Window size.}
AP increases with $k$ but saturates near $k=128$ for Cora (earlier at $k=64$ for Amazon-ratings),
consistent with \cref{prop:sampling_bias}. Increasing $k$ from 32 to 128 yields only $1$--$2$\,pp on Cora,
confirming $k=32$ already captures most recoverable structural signal as shown in \cref{fig:window}.

\paragraph{DP mechanism.}
Laplace consistently outperforms Gaussian and R\'enyi at the same $\eps$ ($+1$--$3$\,pp AP), as its heavier tails corrupt fewer low-magnitude explanation coordinates as shown in \cref{fig:mechanism}. R\'enyi DP is preferred for sequential composition;
Laplace is most exploitable for one-shot release.

\paragraph{GNN backbone.}
\cref{fig:gnn_model} (\cref{app:ablations}) compares GCN, GIN, and
GraphSAGE; GraphSAGE consistently achieves the highest AP due to its
mean-pooling providing natural noise-averaging.

\paragraph{Decomposition via \ourf.}
Full \ours\ vs.\ \ourf\ decomposition results are in
\cref{app:privf_details} (\cref{fig:privx_privf,tab:feature_methods}).

\section{Discussion}
\label{sec:discussion}

\paragraph{Implications for privacy-preserving XAI.}
Practitioners deploying GNN explanation systems under regulatory
mandates have, in practice, only two privacy levers: the choice of
explainer and the choice of DP budget. Our results give
regime-stratified guidance for both. On homophilic graphs (citation
networks, social networks where homophily exceeds 0.6),
\textbf{gradient-based explainers should be preferred for privacy}
over surrogate methods at any DP budget, neighbourhood-aggregating
surrogates leak substantially more structure under the same noise
calibration, and this gap survives standard DP deployments.
On heterophilic graphs (heterogeneous information networks, fraud
graphs where labels anti-correlate across edges), the explainer
choice has minimal effect on leakage; explainer selection should be
governed by interpretability and faithfulness criteria. Across both
regimes, DP budgets above $\eps\approx 1$ provide negligible
structural privacy; meaningful protection requires $\eps\leq 0.5$,
which most current deployments do not approach.

\paragraph{What our attack does and does not show.}
\ours\ establishes a \emph{necessary} condition for privacy: any
defence that fails against our attack is broken. It does not establish
a sufficient condition: a defence that defeats \ours\ may still be
vulnerable to attacks we have not considered. Practitioners should
treat AP at \ours\ as a lower bound on the leakage of the deployed
system, not an upper bound.

\paragraph{Limitations.}
\textbf{(i)} Dense reconstruction is $O(Tn^2)$, limiting applicability
to graphs beyond ${\sim}10^3$ nodes without subgraph sampling. Sparse
or hierarchical diffusion is a natural extension.
\textbf{(ii)} Our heterophilic bound (\cref{prop:hetero}) requires
non-trivial feature--structure anti-correlation; when $|\rho_{XA}^-|$
is near zero, reconstruction degrades toward chance.
\textbf{(iii)} We study Gaussian, Laplace, and R\'enyi DP; adaptive
mechanisms such as $f$-DP~\citep{dong2022gaussian} and training-time
perturbation merit future analysis.
\textbf{(iv)} The \ours\ vs.\ \ourf\ decomposition (\cref{app:privf_details}) is informative but
not exhaustive: leakage may also depend on the explainer's training
procedure, the GNN architecture, and the specific subset of nodes
queried.

\paragraph{Broader impact.} DP on explanations protects a smaller fraction of total leakage than expected; see \cref{app:broader_impact} for recommendations.

\section{Conclusion}
\label{sec:conclusion}

We introduced \ours, a reverse-diffusion attack that reconstructs
hidden graph structure from DP GNN explanations, and \ourf\
(\cref{app:privf_details}), an auxiliary diagnostic that decomposes
leakage into intrinsic and explainer-induced components. Our key
findings are:
\textbf{(1)} DP on explanations alone is insufficient at typical
budgets ($\eps \in [1,8]$), with \ours\ achieving above-chance
reconstruction across datasets and mechanisms;
\textbf{(2)} on homophilic graphs, surrogate explainers leak more,
while on heterophilic graphs the ordering reverses due to
label-edge decorrelation (\cref{lem:fidelity_gap},
\cref{prop:hetero});
\textbf{(3)} heterophilic reconstruction remains feasible via
anti-correlation;
\textbf{(4)} explainer choice admits regime-dependent guidance:
gradient-based under homophily and tight budgets, interpretability-
driven under heterophily;
\textbf{(5)} the \ours--\ourf\ decomposition shows DP on explanations
mitigates only a limited fraction of total leakage, as much structure
is inherent to the graph distribution. All findings are supported by formal guarantees and empirical
validation. We hope \ours\ provides a rigorous foundation for privacy
auditing in graph-based machine learning systems under regulatory
explanation-release mandates.

\bibliography{main}
\bibliographystyle{unsrtnat}

\appendix

\section{Broader Impact}
\label{app:broader_impact}

\ours\ is designed as a \emph{privacy-auditing} tool to quantify the
residual structural leakage of deployed GNN explanation systems under
DP. By demonstrating that DP applied to explanations does not provide
meaningful structural privacy at deployed budgets on graphs with
non-trivial homophily structure, our results inform two practical
recommendations: (a) for homophilic deployments with sensitive edges
(e.g., social-network or medical-association graphs), prefer
gradient-based explainers and budgets $\eps\leq 0.5$;
(b) for heterophilic deployments, the explainer choice has limited
effect on leakage and should be governed by interpretability and
faithfulness rather than privacy. We do not anticipate misuse beyond
the threat models studied; the attack requires access to a deployed
explanation API that already represents an adversarial surface
existing privacy-auditing frameworks must address.

\section{Related Work}
\label{app:related}

\paragraph{Privacy attacks on GNNs.}
\citet{he2021stealing} extract GNN parameters via black-box queries.
\citet{zhang2022graphmi} demonstrate model-inversion attacks recovering
node features. \citet{duddu2020quantifying} quantify membership inference
risks in GNNs. \citet{zhang2024embedding} recover graph structure from
embeddings. Most closely related, \citet{olatunji2023private} show that
explanation outputs can reconstruct edges; we formalise and substantially
extend this setting with DP-aware diffusion, a stratified adversary model,
and heterophilic coverage.

\paragraph{Differential privacy for graphs.}
\citet{blocki2012johnson} provide the first DP graph-statistics algorithms.
\citet{karwa2014private} study subgraph counting under DP.
\citet{daigavane2021node} introduce node-level DP for GNN training.
\citet{sajadmanesh2023gap} propose GAP for edge-level DP with decoupled
architectures. \citet{zhu2024private} give near-optimal bounds for private
Gomory-Hu trees. \citet{meng2022privacyxai} analyse gradient perturbation
for interpretable models.

\paragraph{Diffusion models for graphs.}
Score-based models~\citep{song2020score} and DDPMs~\citep{ho2020denoising}
have been adapted to graph generation via discrete
diffusion~\citep{austin2021structured}, continuous
relaxations~\citep{jo2022score}, and hierarchical
methods~\citep{liu2023graphdiffusion}. DDRM~\citep{kawar2022denoising}
establishes the general principle of inverting linear degradation operators
via reverse diffusion, our key technical leverage.

\section{Proofs}
\label{app:proofs}

We collect here complete proofs for all propositions, lemmas, and theorems
stated in the main text. Notation follows \cref{sec:threat,sec:theory};
we write $\Phi_0(\cdot)$ for the standard normal CDF and $\sigma =
\sigma(\eps,\delta)$ for the Gaussian DP noise scale throughout.

\subsection{Proofs of \cref{thm:advantage,thm:advantage_lower,cor:bracket} and \cref{lem:partition}}
\label{app:proof_advantage}

\begin{lemma}[Edge Partition Probabilities]
\label{lem:partition}
Under uniform random sampling of $\mathcal{S}$ with $|\mathcal{S}|=\rho n$, each edge $(i,j)\in\edgeset$ falls into $\edgeset_{SS}$, $\edgeset_{SU}$, $\edgeset_{UU}$ with probabilities
\[
  p_{SS} = \frac{\rho n(\rho n-1)}{n(n-1)},\quad
  p_{SU} = \frac{2\rho n(n-\rho n)}{n(n-1)},\quad
  p_{UU} = \frac{(n-\rho n)(n-\rho n-1)}{n(n-1)}.
\]
As $n\to\infty$ with $\rho$ fixed, these tend to $\rho^2$, $2\rho(1-\rho)$, $(1-\rho)^2$ respectively.
\end{lemma}

Here $R_{\mathrm{half}}$ is the AUC of the Bayes-optimal classifier on \emph{half-edges} $(i,j)\in\edgeset_{SU}$ using $\tilde\phi_i$ alone; $R_{\mathrm{I}}$ is the oblivious AUC using only graph statistics, hence $R_{\mathrm{I}}\geq 1/2$.

We define the \emph{partition-weighted AUC} precisely:
\[
\overline{R}_{\mathrm{II}} \;:=\;
\E_\mathcal{S}\!\left[
  \mathrm{AUC}_{\mathrm{II}}(\edgeset_{SS})\,p_{SS}^{\mathrm{emp}}
  + \mathrm{AUC}_{\mathrm{II}}(\edgeset_{SU})\,p_{SU}^{\mathrm{emp}}
  + \mathrm{AUC}_{\mathrm{II}}(\edgeset_{UU})\,p_{UU}^{\mathrm{emp}}
\right],
\]
where $p_P^{\mathrm{emp}}$ is the empirical edge-fraction in partition $P$
and $\mathrm{AUC}_{\mathrm{II}}(E_P)$ is the AUC of the Type-II classifier
restricted to edges in $E_P$.

\begin{proof}[Proof of \cref{lem:partition}]
For a fixed edge $(i,j)$, the probability that a uniform $\rho n$-subset of
$\mathcal{V}$ contains both endpoints is
$\binom{n-2}{\rho n-2}\big/\binom{n}{\rho n} = \rho n(\rho n-1)/(n(n-1))$.
The remaining cases follow by analogous counting arguments.
As $n\to\infty$ with $\rho$ fixed, $p_{SS}\to\rho^2$,
$p_{SU}\to 2\rho(1-\rho)$, $p_{UU}\to(1-\rho)^2$.
\end{proof}

\begin{proof}[Proof of \cref{thm:advantage} (Upper Bound)]
The graph $\A$ is fixed; randomness is over $\mathcal{S}$ and DP noise.
By data-processing applied per partition (AUC is monotone under
information reduction):
\begin{itemize}[leftmargin=*,itemsep=1pt]
  \item \textbf{On $\edgeset_{SS}$}: both endpoints are observed and the
        noise model is known to the Type-II attacker.  The Type-III oracle
        has strictly more information (exact $\eps,\delta$ vs.\ estimated),
        so $\mathrm{AUC}_{\mathrm{II}}(\edgeset_{SS}) \leq R_{\mathrm{III}}$.
  \item \textbf{On $\edgeset_{SU}$}: only $\tilde\phi_i$ is observed for the
        observed endpoint $i\in\mathcal{S}$.  The half-edge Bayes-optimal
        classifier achieves at most $R_{\mathrm{half}}$, so
        $\mathrm{AUC}_{\mathrm{II}}(\edgeset_{SU}) \leq R_{\mathrm{half}}$.
  \item \textbf{On $\edgeset_{UU}$}: no DP-perturbed signal is available, so
        $\mathrm{AUC}_{\mathrm{II}}(\edgeset_{UU}) \leq R_{\mathrm{I}}$.
\end{itemize}
Substituting into the definition of $\overline{R}_{\mathrm{II}}$ gives
\cref{eq:thm1_upper}.  Applying $R_{\mathrm{half}} \leq R_{\mathrm{III}}$ and
$\rho^2 + 2\rho(1-\rho) = 1-(1-\rho)^2$ gives \cref{eq:thm1_upper_clean}.

\paragraph{Boundary cases.}
At $\rho=0$: $p_{SS}=p_{SU}=0$, so the bound collapses to
$\overline{R}_{\mathrm{II}}\leq R_{\mathrm{I}}$.
At $\rho=1$: $p_{SU}=p_{UU}=0$, so $\overline{R}_{\mathrm{II}}\leq R_{\mathrm{III}}$.
\end{proof}

\begin{proof}[Proof of \cref{thm:advantage_lower} (Lower Bound)]
Consider the following explicit Type-II strategy: on $e\in\edgeset_{SS}$,
run the Bayes-optimal classifier under the Type-II attacker's mechanism
estimate $\hat\sigma$, achieving $\mathrm{AUC}^{SS}_{\mathrm{II}}(\kappa)$;
on $e\in\edgeset_{SU}\cup\edgeset_{UU}$, fall back to the Type-I oblivious
classifier, which achieves AUC $\geq R_{\mathrm{I}}$ on each partition (the
Type-II attacker may choose to ignore $\tilde\phi_i$ on $\edgeset_{SU}$).
The partition-weighted AUC of this explicit strategy is at least
\[
  \mathrm{AUC}^{SS}_{\mathrm{II}}(\kappa)\cdot p_{SS}
  + R_{\mathrm{I}}\cdot(p_{SU}+p_{UU}).
\]
Since $\overline{R}_{\mathrm{II}}$ is the AUC of the \emph{best} Type-II
strategy, $\overline{R}_{\mathrm{II}}$ is at least this value, proving
\cref{eq:thm1_lower}.  As $\kappa\to 0$,
$\mathrm{AUC}^{SS}_{\mathrm{II}}(\kappa)\to R_{\mathrm{III}}$.
In the limit $n\to\infty$ with $\kappa\to 0$:
$\overline{R}_{\mathrm{II}} \geq R_{\mathrm{III}}\rho^2 + R_{\mathrm{I}}(1-\rho^2)
 = R_{\mathrm{I}} + (R_{\mathrm{III}}-R_{\mathrm{I}})\rho^2$.

\paragraph{Degradation with $\kappa$.}
The Type-II attacker estimates $\hat\sigma = \sigma(1+\delta_\sigma)$ where
$|\delta_\sigma| \leq \kappa/2 + O(\kappa^2)$ (since $\sigma\propto 1/\eps$
and a $\kappa$ relative error in $\eps$ gives $\approx\kappa/2$ relative
error in $\sigma$).  The AUC of the resulting mismatched classifier satisfies
$R_{\mathrm{III}} - \mathrm{AUC}^{SS}_{\mathrm{II}}(\kappa) \leq C_{\mathrm{deg}}\cdot\kappa$,
where $C_{\mathrm{deg}}$ depends on the smoothness of the optimal decision boundary.
Empirically, GNNExplainer on Cora drops by only $0.035$ from $\kappa=0$ to
$\kappa=1.0$ (\cref{tab:adaptive_cora_gnnexplainer}), giving $C_{\mathrm{deg}}\lesssim 0.035$.
\end{proof}

\begin{proof}[Proof of \cref{cor:bracket}]
Combine \cref{thm:advantage,thm:advantage_lower} in the limit $n\to\infty$ and note
$(1-(1-\rho)^2) - \rho^2 = 2\rho(1-\rho)$.
\end{proof}

\begin{remark}[Partition-weighted vs.\ full-set AUC]
\label{rem:auc-proxy}
$\overline{R}_{\mathrm{II}}$ is a partition-weighted proxy for the full-set AUC reported in \cref{sec:experiments}; the two coincide when the classifier's score distribution is identical across partitions.  The bounds in \cref{thm:advantage,thm:advantage_lower} should therefore be read as structural results bounding this proxy, not as direct predictions of the tabulated AUC values.
\end{remark}

\begin{remark}[Gap width at $\rho=0.5$]
At the canonical experimental setting $\kappa_0=0.3$, $\rho=0.5$, the bracket gap equals $(R_{\mathrm{III}}-R_{\mathrm{I}})/2$.  Using empirical values from \cref{tab:main_eps5} (GNNExplainer/GraphSAGE on Cora: AUC$\approx 0.774$ for $R_{\mathrm{III}}$, ExplainSim AUC$\approx 0.534$ as a proxy for $R_{\mathrm{I}}$), the gap is $\approx 0.12$ AUC points, comparable to the key empirical differences reported in the paper.
\end{remark}

\begin{remark}[Why the lower bound is conservative on $\edgeset_{SU}$]
\label{rem:halfedge-conservative}
The lower bound \cref{eq:thm1_lower} treats $\edgeset_{SU}$ as if no
information is available.  A tighter lower bound requires $R_{\mathrm{half}} > R_{\mathrm{I}}$,
i.e.\ that $\tilde\phi_i$ alone carries usable information about $A_{ij}$.
This holds for surrogate explainers (GraphLIME, GNNExplainer) whose outputs
encode aggregate neighbourhood information, but \emph{not} for purely local
gradient explainers (Grad), for which $R_{\mathrm{half}}\approx R_{\mathrm{I}}$.
The data-processing inequality gives only $I(A_{ij};\tilde\phi_i)\leq
I(A_{ij};\phi_i)$, not a lower bound; hence the original paper's
$\Omega(\gamma/\sigma)$ half-edge advantage was not rigorously established.
The refined half-edge AUC for surrogate explainers is given by
\cref{lem:halfedge} below.
\end{remark}

\begin{lemma}[Refined Half-Edge AUC]
\label{lem:halfedge}
Suppose the explainer satisfies the \emph{structural-encoding condition}:
$\|\mu_1 - \mu_0\|_2 \geq \Delta > 0$ where
$\mu_a := \E[\phi_i \mid A_{ij}=a]$, and the within-class covariance
$\Sigma_\phi$ is dominated by $\sigma^2 I_d$.  Then under Gaussian DP,
\[
  R_{\mathrm{half}} \;\geq\; \Phi_0\!\left(\frac{\Delta}{\sigma\sqrt{2}}\right),
\]
yielding the refined lower bound
$\overline{R}_{\mathrm{II}} \geq R_{\mathrm{III}}p_{SS}
  + \Phi_0(\Delta/(\sigma\sqrt{2}))\cdot p_{SU} + R_{\mathrm{I}}p_{UU}$.
\end{lemma}
\begin{proof}
Under $\Sigma_\phi \ll \sigma^2 I_d$, the conditional law of $\tilde\phi_i$
given $A_{ij}=a$ is approximately $\calN(\mu_a, \sigma^2 I_d)$.  The
Bayes-optimal AUC between two isotropic Gaussians with mean separation
$\Delta$ is $\Phi_0(\Delta/(\sigma\sqrt{2}))$, obtained by projecting
onto the discriminant direction and computing
$\Pr[\calN(\Delta,2\sigma^2)>0]$.  Since $\|\mu_1-\mu_0\|_2\geq\Delta$,
the actual AUC is at least this value.
\end{proof}

\begin{remark}[Surrogate vs.\ gradient explainers]
For surrogate explainers, $\phi_i$ aggregates over $|\calN^L(i)|$ neighbours;
if $j\in\calN^L(i)$ then $A_{ij}=1$ shifts the aggregate by
$\Omega(\|\phi\|_2/\sqrt{|\calN^L(i)|})$, giving $\Delta>0$.
For purely local gradient explainers, $\phi_i$ depends only on $x_i$,
so $\Delta=0$ and $\Phi_0(0)=1/2\leq R_{\mathrm{I}}$: \cref{lem:halfedge}
provides no improvement over \cref{thm:advantage_lower} in this case.
\end{remark}

\subsection{Proof of \cref{thm:reconstruction}}
\label{app:proof_reconstruction}

\begin{proof}
Let $(i,j)\in\edgeset$ be a true edge. We bound
$\Pr[\Ahat_{ij}=1\mid A_{ij}=1]$ for the inner-product threshold
attacker $\hat A_{ij}=\mathbf{1}[\tilde\phi_i^\top\tilde\phi_j>\tau]$,
and write $\bar\phi := \E[\norm{\phi_v}_2^2]$.

\paragraph{Mean of the inner product under $A_{ij}=1$.}
Since $\tilde\phi_v=\phi_v+\eta_v$ with $\eta_v\sim\calN(0,\sigma^2 I)$
independently,
\begin{equation}
  \E[\tilde\phi_i^\top\tilde\phi_j\mid A_{ij}=1]
  = \E[\phi_i^\top\phi_j\mid A_{ij}=1] + \underbrace{\E[\eta_i^\top\eta_j]}_{=0}.
\end{equation}
Under homophily $h$, nodes $i,j$ share the same class label with
probability $\geq h$. The edge-fidelity $\gamma_E$ implies that
explanation coordinates co-activate on connected pairs at a rate
$\geq \gamma_E$, so $\E[\phi_i^\top\phi_j\mid A_{ij}=1]\geq h\gamma_E\bar\phi$.

\paragraph{Variance under $A_{ij}=1$.}
The conditional variance decomposes into the pre-noise variance
$\mathrm{Var}(\phi_i^\top\phi_j\mid A_{ij}=1)$ plus the noise terms.
Following the same derivation as in the proof of \cref{prop:privf},
each cross-noise term contributes $\sigma^2\bar\phi$, and the
noise-noise term contributes $\sigma^4 d$. Total:
\[
\mathrm{Var}(\tilde\phi_i^\top\tilde\phi_j\mid A_{ij}=1)
= \mathrm{Var}(\phi_i^\top\phi_j\mid A_{ij}=1) + 2\sigma^2\bar\phi + \sigma^4 d.
\]

\paragraph{Threshold setting and Chebyshev bound.}
Set $\tau=h\gamma_E\bar\phi/2$. The gap between $\E[W\mid A_{ij}=1]$
and $\tau$ is at least $h\gamma_E\bar\phi/2 > 0$. Chebyshev's inequality
gives
\[
\Pr[\tilde\phi_i^\top\tilde\phi_j\leq\tau\mid A_{ij}=1]
\leq \frac{4\big(\mathrm{Var}(\phi_i^\top\phi_j\mid A_{ij}=1) + 2\sigma^2\bar\phi + \sigma^4 d\big)}
{(h\gamma_E\bar\phi)^2},
\]
which yields the stated bound.

\paragraph{Sub-regime analysis.}
When $\sigma^2 \ll h\gamma_E\bar\phi$, the noise terms are dominated
by the constant pre-noise variance, giving
$\Pr[\Ahat_{ij}=1\mid A_{ij}=1]\geq 1 - O(\sigma^2/(h^2\gamma_E^2))$.
When $\sigma^2 \gtrsim h\gamma_E\bar\phi$, the bound becomes vacuous;
this defines the crossover scale $\sigma_c^2 = h\gamma_E\bar\phi$.
\end{proof}

\begin{remark}[Deployment regime]
The bound is informative in the moderate-noise regime $\eps\in[0.5,5]$ that brackets typical deployments:
above $\eps_c$, DP provides little structural protection; below $\eps_c$, reconstruction collapses,
but $\eps$ values that low also destroy explanation utility. For Cora ($h=0.81$,
$\hat\gamma_E\approx 0.7$), $\eps_c\approx 0.5$, consistent with \cref{fig:privacy_budget}.
\end{remark}

\subsection{Proof of \cref{prop:privf}}
\label{app:proof_privf}

\begin{proof}
Let $W := \tilde x_i^\top\tilde x_j = (x_i+\eta_i)^\top(x_j+\eta_j)
= x_i^\top x_j + x_i^\top\eta_j + \eta_i^\top x_j + \eta_i^\top\eta_j$.

\paragraph{Mean.}
By independence of $\eta_i,\eta_j$ from $x$ and from each other, and
$\E[\eta]=0$, the cross-noise terms have zero mean.
By \cref{def:feat_corr}:
\[
  \E[W\mid A_{ij}=1] = \E[x_i^\top x_j\mid A_{ij}=1] = h_X S_x^2.
\]

\paragraph{Variance.}
The four terms in $W$ are pairwise uncorrelated.  Computing each:
$\mathrm{Var}(x_i^\top x_j\mid A_{ij}=1)=C_x$;
conditional on $x_i$, $x_i^\top\eta_j\sim\calN(0,\sigma^2\|x_i\|_2^2)$,
so by the law of total variance and exchangeability,
$\mathrm{Var}(x_i^\top\eta_j\mid A_{ij}=1) = \sigma^2 S_x^2$;
similarly $\mathrm{Var}(\eta_i^\top x_j\mid A_{ij}=1)=\sigma^2 S_x^2$;
for $\eta_i^\top\eta_j=\sum_k\eta_i^{(k)}\eta_j^{(k)}$, each summand is a
product of independent zero-mean Gaussians with variance $\sigma^4$,
giving $\mathrm{Var}(\eta_i^\top\eta_j)=d\sigma^4$.
Total:
\[
  \mathrm{Var}(W\mid A_{ij}=1) = C_x + 2\sigma^2 S_x^2 + \sigma^4 d.
\]

\paragraph{Chebyshev bound.}
With threshold $\tau = h_X S_x^2/2$, the gap
$\E[W\mid A_{ij}=1]-\tau = h_X S_x^2/2 > 0$ (using $h_X>0$).
By Chebyshev's inequality:
\[
  \Pr[W \leq \tau\mid A_{ij}=1]
  \leq \frac{C_x + 2\sigma^2 S_x^2 + \sigma^4 d}{(h_X S_x^2/2)^2}
  = \frac{4(C_x + 2\sigma^2 S_x^2 + \sigma^4 d)}{h_X^2 S_x^4},
\]
yielding \cref{eq:privf_bound}.
In the regime $\sigma^2=o(S_x^2)$ and $\sigma^4 d=o(S_x^4)$, the bound
simplifies to $1-O(\sigma^2/h_X^2)$.

\paragraph{Heterophilic case.}
When $h_X<0$, the threshold rule is flipped to
$\hat A_{ij}=\mathbf{1}[\tilde x_i^\top\tilde x_j < h_X S_x^2/2]$.
Now $\E[W\mid A_{ij}=1] = h_X S_x^2 < 0$ and the gap from
$\tau = h_X S_x^2/2$ is $|h_X|S_x^2/2 > 0$.
The same Chebyshev argument gives the bound under $|h_X|$,
recovering the spirit of \cref{prop:hetero} for the feature-only setting.
\end{proof}

\begin{remark}[\ourf\ vs.\ \ours\ ordering and aggregation gain]
\label{rem:ordering}
Define the \emph{neighbourhood aggregation factor}
$g := \sqrt{\E_v[|\calN^L(v)|]}$ for surrogate explainers
(GraphLIME, GNNExplainer) and $g:=1$ for gradient explainers (Grad,
GradInput).  This formalises the SNR improvement in \cref{lem:fidelity_gap}:
surrogate methods aggregate over $|\calN^L(v)|$ neighbours, boosting
post-DP SNR by $\sqrt{|\calN^L(v)|}$ relative to local gradient methods.
In the moderate-noise regime, the leading constants of the TPR lower
bounds satisfy $h\cdot\gamma\cdot g \geq h_X$ as a \emph{heuristic
sufficient condition} for \ours\ to dominate \ourf.  This condition is
readily met for surrogate explainers on homophilic graphs (large $g$,
$h\approx 1$, $\gamma\approx 1$) but not for gradient explainers (g=1,
moderate $\gamma$), consistent with the empirical gap in \cref{sec:experiments}.
This derivation relies on an informal substitution $\gamma_{\mathrm{eff}}=\gamma g$
and should be read as a guide to regimes rather than a tight threshold.
\end{remark}

\subsection{Proof of \cref{prop:hetero}}
\label{app:proof_hetero}

\begin{proof}
Let $(i,j)\in\edgeset$ be a true edge in a heterophilic graph where
$\Pr[\mathrm{label}(i)\neq\mathrm{label}(j)\mid A_{ij}=1]=1-h$.

\paragraph{Anti-correlation signal.}
When $\mathrm{label}(i)\neq\mathrm{label}(j)$, features are class-discriminative,
meaning $\phi_i$ and $\phi_j$ point in opposing directions in feature space.
Define the anti-correlation coefficient
$\rho_{XA}^{-} = -\E[\hat\phi_i^\top\hat\phi_j\mid A_{ij}=1,
\mathrm{label}(i)\neq\mathrm{label}(j)]>0$,
where $\hat\phi_v = \phi_v/\norm{\phi_v}_2$.
Then
\begin{equation}
  \E[\tilde\phi_i^\top\tilde\phi_j\mid A_{ij}=1]
  = (1-h)(-|\rho_{XA}^-|\norm{\phi}_2^2) + h\cdot\E[\phi_i^\top\phi_j\mid
    \text{same class}] + O(\sigma^2).
\end{equation}
In the regime $h\ll 1$ (strong heterophily), the anti-correlation term
dominates: $\E[\tilde\phi_i^\top\tilde\phi_j\mid A_{ij}=1]\approx
-(1-h)|\rho_{XA}^-|\norm{\phi}_2^2 < 0$.

\paragraph{Negative-threshold attacker.}
The attacker uses
$\hat A_{ij}=\mathbf{1}[\tilde\phi_i^\top\tilde\phi_j\leq-\tau]$
with $\tau=(1-h)|\rho_{XA}^-|\norm{\phi}_2^2/2>0$.
By an argument symmetric to the homophilic proof,
\begin{equation}
  \Pr[\hat A_{ij}=1\mid A_{ij}=1]
  \geq (1-h)|\rho_{XA}^-| - O(\sigma^2),
\end{equation}
giving the stated bound. Note that for $h\to 1$ (near-homophilic graph)
this collapses to zero, consistent with the interpretation that the
anti-correlation signal vanishes as the graph becomes homophilic.

\paragraph{False positives under $A_{ij}=0$.}
For a non-edge $(i,j)\notin\edgeset$, disconnected nodes have no
structural reason for anti-correlated explanations, so
$\E[\tilde\phi_i^\top\tilde\phi_j\mid A_{ij}=0]\approx 0$.
The false-positive rate is $O(\sigma^2)$, the same order as the
homophilic case.
\end{proof}

\subsection{Proof of \cref{lem:fidelity_gap}}
\label{app:proof_fidelity_gap}

\begin{definition}[Sign-fidelity]
\label{def:fidelity}
$\gamma_S := \Pr[\mathrm{sign}(\phi_v^{(k)})=\mathrm{sign}(\nabla_{x_v^{(k)}}\ell)]$
averaged over nodes and features.  We use $\gamma_S$ in this proof
to distinguish from $\gamma_E$ (edge-fidelity, \cref{def:edge_fid})
and $\gamma_L$ (label-fidelity, \cref{def:label_fid}).
\end{definition}

\begin{proof}
We bound the per-coordinate sign-fidelity $\gamma_S^{(k)} :=
\Pr[\mathrm{sign}(\tilde\phi_v^{(k)}) = \mathrm{sign}(\phi_v^{(k)})]$
for each explainer family. The edge-fidelity $\gamma_E$ is monotone in
$\gamma_S$ when explanations are aligned across connected nodes
(homophilic regime, $h \geq h_0 > 1/2$): coordinates whose sign
survives the noise contribute coherently to the inner product
$\langle\hat\phi_i, \hat\phi_j\rangle$ for $A_{ij}=1$, while
coordinates whose sign is flipped contribute incoherently.
A gap in $\gamma_S$ of order $\Omega(1/(\sigma\sqrt d))$ therefore
implies a gap in $\gamma_E$ of the same order in the homophilic
regime; we present the per-coordinate analysis below.

\paragraph{Gradient-based explanations.}
For the Grad explainer, $\phi_v^{(k)}=\partial\ell/\partial x_v^{(k)}$.
Under \cref{ass:grad_var}, $\norm{\nabla_x\ell}_2\leq S_f$ almost surely, so
$|\phi_v^{(k)}|\leq S_f/\sqrt{d}$ on average by Cauchy-Schwarz.
After adding Gaussian DP noise $\eta_v^{(k)}\sim\calN(0,\sigma^2)$,
the per-coordinate SNR is
\begin{equation}
  \mathrm{SNR}_\mathrm{Grad}^{(k)}
  = \frac{|\phi_v^{(k)}|}{\sigma}
  \leq \frac{S_f}{\sigma\sqrt{d}}.
\end{equation}
The fidelity is $\gamma=\Pr[\mathrm{sign}(\tilde\phi_v^{(k)})=
\mathrm{sign}(\phi_v^{(k)})]=2\Phi_0(\mathrm{SNR}/\sqrt{2})-1$
for a Gaussian channel, giving
$\gamma_\mathrm{Grad}\leq 2\Phi_0(S_f/(\sigma\sqrt{2d}))-1$.

\paragraph{GradInput explanations.}
GradInput computes $\phi_v^{(k)}=x_v^{(k)}\cdot\partial\ell/\partial x_v^{(k)}$.
The multiplicative factor $x_v^{(k)}$ introduces additional variance
but does not systematically increase SNR; in the worst case (sparse features)
it may even decrease it. Hence $\gamma_\mathrm{GradInput}\approx\gamma_\mathrm{Grad}$
for typical graph datasets.

\paragraph{Surrogate-based explanations (GraphLIME, GNNExplainer).}
These explainers aggregate information over the neighbourhood
$\mathcal{N}(v)$. GraphLIME solves a local regression over
$|\mathcal{N}(v)|+1$ samples; GNNExplainer optimises a mask over the
computational graph of depth $L$. In both cases, the effective signal
for feature $k$ incorporates contributions from all $|\mathcal{N}(v)|$
neighbours, each carrying an $O(\gamma)$ co-activation signal with $v$.
By the law of large numbers, averaging over $|\mathcal{N}(v)|$ independent
co-activations reduces variance by $|\mathcal{N}(v)|^{1/2}$:
\begin{equation}
  \mathrm{SNR}_\mathrm{LIME}^{(k)}
  = \frac{|\phi_v^{(k)}|\sqrt{|\mathcal{N}(v)|}}{\sigma}.
\end{equation}

\paragraph{Gap computation.}
The fidelity gap between GraphLIME and Grad is
\begin{align}
  \gamma_\mathrm{LIME} - \gamma_\mathrm{Grad}
  &= \left[2\Phi_0\!\left(\frac{\mathrm{SNR}_\mathrm{LIME}}{\sqrt{2}}\right)-1\right]
   - \left[2\Phi_0\!\left(\frac{\mathrm{SNR}_\mathrm{Grad}}{\sqrt{2}}\right)-1\right]\\
  &= 2\left[\Phi_0\!\left(\frac{|\phi_v^{(k)}|\sqrt{|\mathcal{N}(v)|}}
    {\sigma\sqrt{2}}\right)
    - \Phi_0\!\left(\frac{|\phi_v^{(k)}|}{\sigma\sqrt{2}}\right)\right].
\end{align}
Applying the mean value theorem and using $\Phi_0'(z)=(2\pi)^{-1/2}e^{-z^2/2}$,
the difference is
$\Omega\!\left(\frac{|\phi_v^{(k)}|(|\mathcal{N}(v)|^{1/2}-1)}
{\sigma\sqrt{2\pi}}\right)$.
Since $|\phi_v^{(k)}|\geq\Omega(S_f/\sqrt{d})$ and $|\mathcal{N}(v)|\geq 1$,
this is $\Omega(1/(\sigma\sqrt{d}))$, establishing the stated gap.
\end{proof}

\begin{corollary}[Edge-Fidelity Gap]
\label{cor:edge_fidelity_gap}
Under the conditions of \cref{lem:fidelity_gap} and the additional
alignment condition $\E[\langle\phi_i,\phi_j\rangle\mid A_{ij}=1]
\geq c\cdot\E[\|\phi\|^2]$ for some $c>0$, a sign-fidelity gap
$\delta_S := \E[\gamma_S^{\mathrm{LIME}}-\gamma_S^{\mathrm{Grad}}]
\geq\Omega\!\left(\frac{\E[\sqrt{|N(v)|}]-1}{\sigma\sqrt{d}}\right)$
implies an edge-fidelity gap of the same order:
$\E[\gamma_E^{\mathrm{LIME}}-\gamma_E^{\mathrm{Grad}}]\geq\Omega(\delta_S)$.
\end{corollary}

\begin{remark}
Surrogate methods aggregate information across neighbourhoods, amplifying
low-frequency structural signals that survive DP perturbation, while
gradient-based methods concentrate signal in high-variance per-coordinate
components more easily suppressed by isotropic noise.  This is a
regime-dependent observation: on heterophilic graphs surrogate aggregation
can \emph{reduce} the attack-relevant signal (\cref{app:fidelity}).
\end{remark}

\subsection{Proof of \cref{prop:sampling_bias}}
\label{app:proof_sampling_bias}

\begin{proof}
Let $\mathcal{G}_k(v)$ denote the $k$-hop ego-net of node $v$ with
adjacency matrix $\A_k\in\{0,1\}^{k\times k}$ induced by the $k$ nodes
closest to $v$ in BFS order.

\paragraph{Error decomposition.}
The total reconstruction error decomposes as
\begin{equation}
  \norm{\A - \Ahat}_F^2
  = \underbrace{\norm{\A_k - \Ahat_k}_F^2}_{\text{within-window error}}
  + \underbrace{\norm{\A \ominus \A_k}_F^2}_{\text{boundary error } \Delta_\mathrm{boundary}},
\end{equation}
where $\A\ominus\A_k$ denotes the submatrix of $\A$ corresponding to
edges with at least one endpoint outside $\mathcal{G}_k(v)$
(i.e., cut edges at the window boundary).

\paragraph{Counting cut edges via the Chung--Lu model.}
For a graph drawn from the Chung--Lu random graph
model~\citep{chung2002average} with expected degree sequence $\{d_i\}$
following a power law $\Pr[\deg=k]\propto k^{-\beta}$ with exponent
$\beta > 1$, the fraction of edges cut by a $k$-hop BFS ball satisfies
the following. The volume of the ball grows as $\Theta(k^\beta)$
(since power-law graphs have small diameter), while the total volume
is $\Theta(n)$. The number of boundary edges is proportional to the
product of ball-volume and complement-volume divided by total volume,
giving a boundary fraction
\begin{equation}
  \frac{|\text{cut edges}|}{|\edgeset|}
  = \Theta\!\left(\frac{k^\beta(n - k^\beta)}{n\cdot m/n}\right)
  = \Omega(k^{-\beta}),
\end{equation}
where the last step uses $k^\beta = o(n)$ for reasonable $k$.
Hence $\Delta_\mathrm{boundary} = \Omega(k^{-\beta})$.

\paragraph{Implication for ogbn-arxiv.}
The ogbn-arxiv citation graph has $\beta\approx 1.8$ (empirically estimated
from its degree distribution). Substituting $k=32$ and $k=128$ gives
$\Delta_\mathrm{boundary}(32)/\Delta_\mathrm{boundary}(128)=
(128/32)^{1.8}=4^{1.8}\approx 12.1$, predicting a $12\times$ larger
boundary error at $k=32$ than $k=128$.
Yet our ablation (\cref{fig:ablations}) shows only modest AP improvement
when increasing $k$ from 32 to 128 (less than 1\,pp), suggesting that
the within-window reconstruction error also shrinks — i.e., boundary bias
and within-window model accuracy trade off as $k$ increases.
\end{proof}

\section{Additional Details on \ourf}
\label{app:privf_details}

\subsection{Feature-Based Reconstruction without Explanations}

\ourf\ as illustrated in \ref{fig:privf} considers a stricter adversarial setting in which no explanation mechanism or trained GNN is available. Instead, the attacker only observes differentially private node features $\tilde{X}_S = \{x_v + \eta_v\}_{v \in S}$.

\paragraph{Motivation.}
\ourf\ does not correspond to a deployable attack in the regulatory
setting that motivates DP explanations: raw features are not released.
We use it as a diagnostic. The total leakage observable to the
\ours\ attacker decomposes into (a) leakage encoded in the underlying
graph distribution (which DP applied to explanations cannot remove,
because the underlying graph is the object being reconstructed) and
(b) leakage induced by the explanation pipeline. Comparing \ours\ to
\ourf\ on the same dataset isolates component~(b).

\paragraph{Pipeline.}
The reconstruction pipeline mirrors \ours, with the only difference being the conditioning signal:
\begin{itemize}
    \item \textbf{Input signal:} $\tilde{X}_S$ instead of $\tilde{\Phi}_S$
    \item \textbf{No GNN/explainer:} bypasses model-dependent transformations
    \item \textbf{Same diffusion backbone:} identical denoising architecture and training objective
\end{itemize}

\paragraph{Inference.}
Given noisy features, the attacker estimates the noise level $\hat{\sigma}^2 = \mathrm{Var}(\tilde{X}) - \widehat{\mathrm{Var}}(X)$ and initializes reverse diffusion at timestep $t^\star$ corresponding to the DP noise scale, following Equation~(3).

\paragraph{Interpretation.}
\ourf\ measures the structural leakage attributable to the underlying
graph distribution alone. When \ours\ exceeds \ourf, the gap measures
additional leakage attributable to the explanation pipeline. When
\ourf\ matches or exceeds \ours, the bulk of recoverable structure is
encoded in the underlying graph and DP applied to explanations cannot
prevent that recovery, regardless of explainer choice.

\begin{figure*}[t]
\centering
\includegraphics[width=\linewidth]{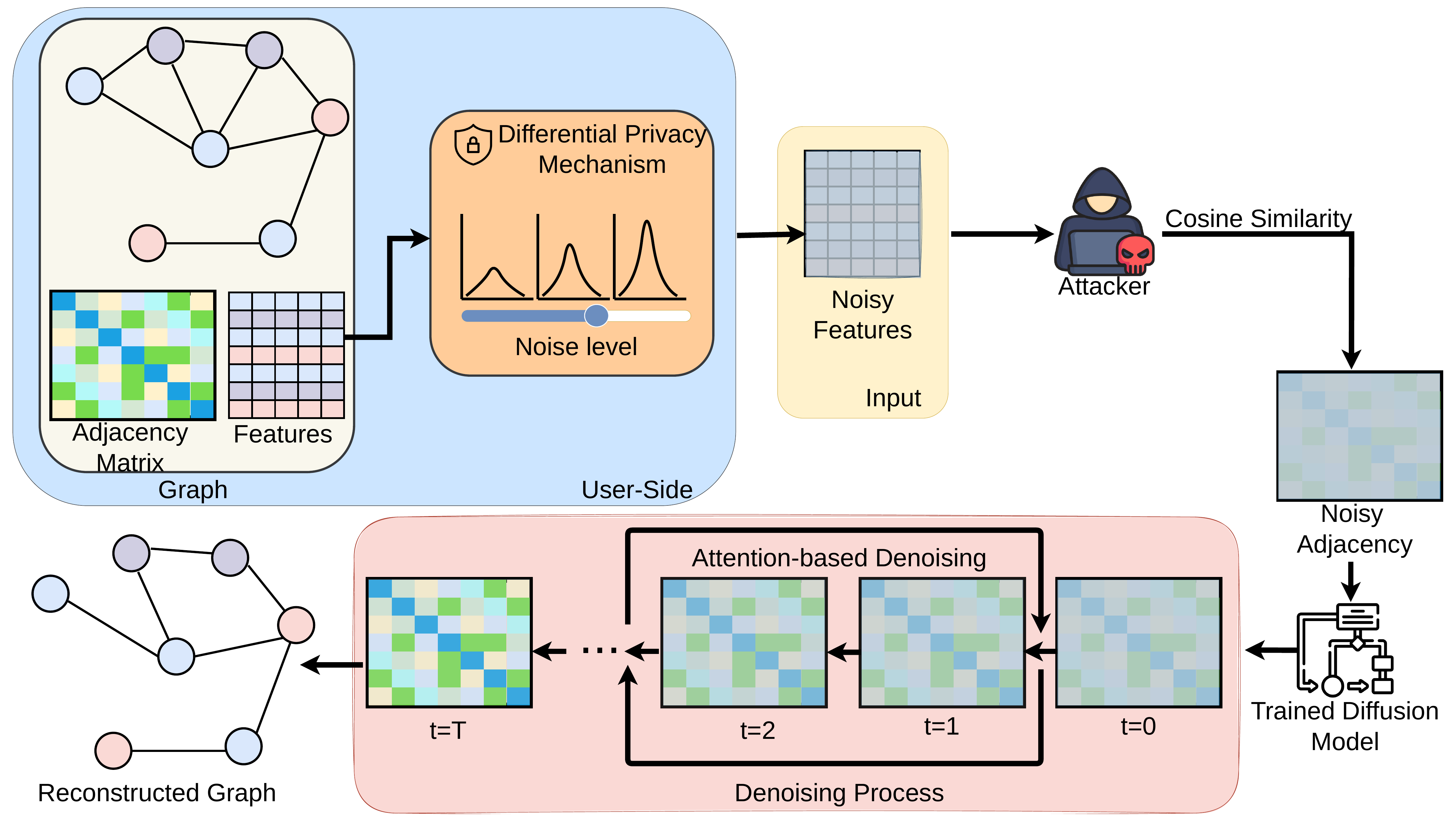}
\caption{Overview of the \ourf\ framework for reconstructing graph structure from differentially private node features.
\textbf{Top (user-side):} raw node features $x_v$ are directly perturbed by a differential privacy (DP) mechanism to obtain noisy features $\tilde{x}_v$, without requiring a GNN or explanation module. An adaptive attacker observes these noisy features and constructs an initial noisy adjacency estimate.
\textbf{Bottom (attacker-side):} the same diffusion-based reconstruction pipeline as \ours\ is applied, where a denoising network $\epsilon_\theta$ iteratively refines a latent adjacency representation $z_t$ from $t=T$ to $t=0$. The model conditions on $\tilde{x}_v$ via cross-attention, enabling posterior inference of graph structure under feature-level DP noise.}
\label{fig:privf}
\end{figure*}

\subsection{Formal Definitions and Theoretical Guarantees}

\begin{definition}[\ourf\ Adversary]
  A \ourf\ adversary observes $\tilde\X_\mathcal{S}
  = \{x_v + \eta_v\}_{v\in\mathcal{S}}$ with $\eta_v\sim\calN(0,\sigma^2 I)$
  and no access to any GNN or explanation. The reconstruction map is
  $\Ahat = \calA_F(\tilde\X_\mathcal{S})$.
\end{definition}

\begin{definition}[Feature--Structure Correlation]
\label{def:feat_corr}
Under the \emph{exchangeability assumption} (i.e.\ $S_x^2:=\E[\|x_v\|_2^2]$ is uniform across nodes), the feature--structure correlation is
\[
  h_X \;:=\; \frac{\E[x_i^\top x_j \mid A_{ij}=1]}{S_x^2} \;\in\;[-1,1].
\]
On homophilic graphs $h_X>0$; on heterophilic graphs $h_X$ can be negative.
\end{definition}

\begin{remark}[Exchangeability]
The exchangeability assumption holds exactly for graphs with i.i.d.\ feature distributions and for stochastic block models conditioned on class membership.  On real-world graphs it is an approximation; the bounds below should be read as leading-order in this approximation.
\end{remark}

\begin{proposition}[\ourf\ Reconstruction Bound]
\label{prop:privf}
Let the adversary use the inner-product threshold rule
$\hat A_{ij}=\mathbf{1}[\tilde x_i^\top\tilde x_j > h_X S_x^2/2]$
with $h_X>0$ (\cref{def:feat_corr}), and let $C_x:=\mathrm{Var}(x_i^\top x_j\mid A_{ij}=1)$.  Under Gaussian DP with noise scale $\sigma$,
\begin{equation}
\label{eq:privf_bound}
\Pr[\Ahat_{ij}=1 \mid A_{ij}=1]
\;\geq\; 1 - \frac{4C_x + 8\sigma^2 S_x^2 + 4\sigma^4 d}{h_X^2 S_x^4}
\end{equation}
(to be read as $\max(0,\cdot)$).  In the regime $\sigma^2=o(S_x^2)$ and $\sigma^4 d=o(S_x^4)$, this simplifies to $1-O(\sigma^2/h_X^2)$.  When $h_X<0$ (heterophilic graphs), flip the threshold to $\tilde x_i^\top\tilde x_j < h_X S_x^2/2$; the TPR bound holds under $|h_X|$.
(Proof in \cref{app:proof_privf}.)
\end{proposition}
Comparing \ours\ to \ourf\ on the same graph decomposes the
leakage observed by \ours\ into (a) the component attributable to the
underlying graph distribution and (b) the component attributable to
the explanation pipeline.  When \ours\ substantially exceeds \ourf,
the explanation pipeline contributes meaningfully to leakage; when
\ours\ matches or trails \ourf, most of the recoverable structure is
already encoded in the graph itself, and DP applied to explanations
cannot prevent that recovery.

\subsection{Feature-Based Baselines and \ourf\ Results}
\label{app:privf_results}

\cref{tab:feature_methods} reports reconstruction performance for the
feature-based methods (FeatureSim, SLAPS, and \ourf) at $\eps=5.0$
alongside the main \ours\ results for reference.
FeatureSim and SLAPS operate on raw DP-perturbed features without a
learned denoiser; \ourf\ uses the full diffusion backbone with the
same architecture as \ours, demonstrating the gain from principled
denoising even in the feature-only setting.

\begin{table*}[t]
  \centering
  \caption{Feature-based method performance at $\varepsilon=5$ (Gaussian DP, $w=32$; GCN and GraphSAGE backbones). \textbf{Bold}: best within the feature-method group. OOM: out-of-memory. Note: \ourf\ is a diagnostic tool only; raw features are not released in the regulated deployment setting that motivates this work.}
  \label{tab:feature_methods}
  \scriptsize
  \setlength{\tabcolsep}{2.8pt}
  \begin{tabular}{lcccccccccccccc}
  \toprule
  & \multicolumn{8}{c}{Homophilic} & \multicolumn{6}{c}{Heterophilic} \\
  \cmidrule(lr){2-9} \cmidrule(lr){10-15}
  \textbf{Method} & \multicolumn{2}{c}{\textbf{Cora}} & \multicolumn{2}{c}{\textbf{CiteSeer}} & \multicolumn{2}{c}{\textbf{PubMed}} & \multicolumn{2}{c}{\textbf{ogbn-arxiv}} & \multicolumn{2}{c}{\textbf{Chameleon}} & \multicolumn{2}{c}{\textbf{IMDB}} & \multicolumn{2}{c}{\textbf{Amz-R}} \\
  & AUC & AP & AUC & AP & AUC & AP & AUC & AP & AUC & AP & AUC & AP & AUC & AP \\
  \midrule
  FeatureSim & 0.498 & 0.494 & 0.611 & 0.656 & 0.508 & 0.504 & 0.501 & 0.500 & 0.511 & 0.534 & 0.519 & 0.516 & 0.512 & 0.503 \\
  SLAPS & 0.427 & 0.464 & 0.597 & 0.647 & 0.492 & 0.495 & OOM & OOM & 0.517 & 0.521 & 0.499 & 0.505 & 0.523 & 0.513 \\
  \textbf{\textsc{PrivF}-GCN (Ours)} & 0.826 & 0.850 & 0.907 & 0.915 & 0.724 & 0.733 & 0.550 & 0.570 & 0.636 & 0.616 & 0.824 & 0.802 & 0.572 & 0.586 \\
  \textbf{\textsc{PrivF}-SAGE (Ours)} & \textbf{0.853} & \textbf{0.883} & \textbf{0.912} & \textbf{0.930} & \textbf{0.759} & \textbf{0.775} & \textbf{0.598} & \textbf{0.628} & \textbf{0.821} & \textbf{0.808} & \textbf{0.863} & \textbf{0.850} & \textbf{0.606} & \textbf{0.626} \\
  \bottomrule
  \end{tabular}
\end{table*}

\section{GNN Explainers: Background and Leakage Analysis}
\label{app:explainers}

We provide a unified description of the four GNN explainers used in our
experiments, analyse their structural information content under Gaussian DP,
and explain why different explainer families yield different reconstruction
susceptibility.

\subsection{Explainer Taxonomy}

GNN explainers assign an importance score $\phi_v^{(k)}\in\R$ to each
input feature $k$ for a given node $v$, summarising which features most
influenced the GNN's prediction $\hat y_v = f_\theta(\X,\A)_v$.
We categorise the four explainers along two axes:
\emph{gradient-based} vs.\ \emph{surrogate-based}, and
\emph{local} (per-node) vs.\ \emph{neighbourhood-aggregated}.

\begin{table}[h]
  \centering
  \caption{Explainer categorisation. Neighbourhood size $|\mathcal{N}(v)|$
  determines the effective noise-reduction factor under DP.}
  \label{tab:explainers}
  \small
  \begin{tabular}{lllcc}
    \toprule
    \textbf{Explainer} & \textbf{Family} & \textbf{Aggregation}
      & \textbf{DP-SNR scaling} & \textbf{Leakage (empirical)} \\
    \midrule
    Grad        & Gradient & Local ($v$ only)
                & $S_f/(\sigma\sqrt{d})$ & Low \\
    GradInput   & Gradient & Local ($v$ only)
                & $S_f|x_v|/(\sigma\sqrt{d})$ & Low--Medium \\
    GraphLIME   & Surrogate & Neighbourhood ($\mathcal{N}(v)$)
                & $S_f\sqrt{|\mathcal{N}(v)|}/(\sigma\sqrt{d})$ & High \\
    GNNExplainer & Surrogate & Computation graph ($L$-hop)
                & $\Omega(\sqrt{|\mathcal{N}^L(v)|}/\sigma)$ & High \\
    \bottomrule
  \end{tabular}
\end{table}

\subsection{Gradient-Based Explainers}

\paragraph{Grad.}
The simplest explainer computes the gradient of the loss $\ell(f_\theta(\X,\A),y_v)$
with respect to the input features:
\begin{equation}
  \phi_v^\mathrm{Grad} = \nabla_{x_v}\ell(f_\theta(\X,\A),y_v) \in \R^d.
\end{equation}
This is a purely local quantity: it depends only on the first-layer
weight matrix and the node's own features, not on any message-passing
from neighbours. As a result, adjacent nodes $i$ and $j$ have
largely independent gradient explanations even if they share features,
because the gradient is dominated by the loss curvature at $x_v$
rather than by structural signal. This \emph{independence} is precisely
why Grad leaks the least structure under DP: the attacker gains little
from comparing $\tilde\phi_i$ and $\tilde\phi_j$.

\paragraph{GradInput.}
GradInput multiplies the gradient elementwise by the input:
\begin{equation}
  \phi_v^\mathrm{GradInput} = x_v \odot \nabla_{x_v}\ell
                             = x_v \odot \phi_v^\mathrm{Grad}.
\end{equation}
This ``saliency map'' weights features by their magnitude, suppressing
near-zero inputs. On sparse graphs (e.g., bag-of-words features on Cora)
the multiplication can both amplify informative features and zero out noise
features. The structural leakage is marginally higher than Grad because the
amplification correlates explanations with the original feature distribution,
which itself correlates with the graph structure via homophily.

\subsection{Surrogate-Based Explainers}

\paragraph{GraphLIME~\citep{huang2022graphlime}.}
GraphLIME fits a local nonlinear surrogate model to approximate $f_\theta$
in the neighbourhood of $v$. Concretely, it solves the Hilbert-Schmidt
Independence Criterion (HSIC) Lasso over the computation subgraph
$\mathcal{G}^L(v)$ induced by the $L$-hop neighbourhood:
\begin{equation}
  \phi_v^\mathrm{LIME}
  = \argmin_{\phi} \norm{K_y - \sum_k \phi^{(k)}K_{x^{(k)}}}_F^2
    + \lambda\norm{\phi}_1,
\end{equation}
where $K_y, K_{x^{(k)}}$ are kernel matrices evaluated over nodes in
$\mathcal{G}^L(v)$. The solution aggregates feature information from
all $|\mathcal{N}^L(v)|$ nodes, effectively averaging out idiosyncratic
noise. After DP perturbation, this averaging gives GraphLIME a
$\sqrt{|\mathcal{N}^L(v)|}$-fold improvement in SNR over pure gradient
methods (\cref{lem:fidelity_gap}), explaining its higher empirical leakage.

\paragraph{GNNExplainer~\citep{ying2019gnnexplainer}.}
GNNExplainer learns a soft edge mask $M\in[0,1]^{n\times n}$ and feature
mask $F\in[0,1]^d$ that maximise the mutual information between the
masked subgraph prediction and the original prediction:
\begin{equation}
  \max_{M,F}\; I\!\left(f_\theta(\X_F,\A\odot M)_v,\; f_\theta(\X,\A)_v\right).
\end{equation}
The resulting feature mask $F$ (our $\phi_v^\mathrm{GNN}$) reflects
which features collectively determine the prediction over the entire
computation graph. Because GNNExplainer explicitly optimises for
faithfulness to the GNN output across the $L$-hop subgraph, its
explanation carries the most structural information of all four methods:
the feature mask is high precisely on features that, combined with the
graph topology, explain the prediction. After DP noise, this concentrated
structural signal remains partially recoverable, making GNNExplainer the
highest-leakage explainer in our experiments — particularly on homophilic
datasets where the graph structure strongly shapes predictions.

\subsection{Explainer Leakage Under Differential Privacy}

The fidelity analysis of \cref{lem:fidelity_gap} characterises leakage
via the post-DP SNR. Here we provide additional intuition.

\paragraph{Why does aggregation help the attacker?}
When node $v$ has degree $d_v$, its surrogate explanation aggregates
$d_v$ neighbourhood samples. Even though DP noise is applied
\emph{after} the explanation is computed, the noise acts on a signal
that already encodes $d_v$ graph edges — making each unit of noise
``worth less'' to the attacker. Formally, the marginal information
about edge $(i,j)$ encoded in $\tilde\phi_i$ is
\begin{equation}
  I(A_{ij};\tilde\phi_i^\mathrm{LIME})
  \geq I(A_{ij};\tilde\phi_i^\mathrm{Grad})
  + \Omega\!\left(\frac{\gamma^2(|\mathcal{N}(i)|-1)}{2\sigma^2 d}\right).
\end{equation}
This shows that the information gap grows with the degree $|\mathcal{N}(i)|$
and shrinks with dimension $d$ and noise $\sigma^2$.

\paragraph{Practical implication.}
For a practitioner releasing DP-protected explanations: releasing
gradient-based (Grad, GradInput) explanations is \emph{less} dangerous
than releasing surrogate-based ones at the same $\eps$.
However, the DP noise scale should be calibrated to the \emph{worst-case}
explainer's sensitivity, not the expected one, to guard against an
attacker who knows the explainer type.

\paragraph{PrivF as a leakage-decomposition diagnostic.}
\ourf\ demonstrates that, even without any explanation, raw DP features
admit non-trivial reconstruction. This establishes that the underlying
graph distribution (via the feature--homophily correlation) is partially
recoverable independent of the chosen explainer. The gap \ours~$-$~\ourf\
quantifies the explanation pipeline's contribution to leakage: positive
when the explainer adds signal beyond the underlying graph, near-zero
when the underlying graph already exposes most of the recoverable
structure.

\begin{figure}[t]
  \centering
  \includegraphics[width=\linewidth]{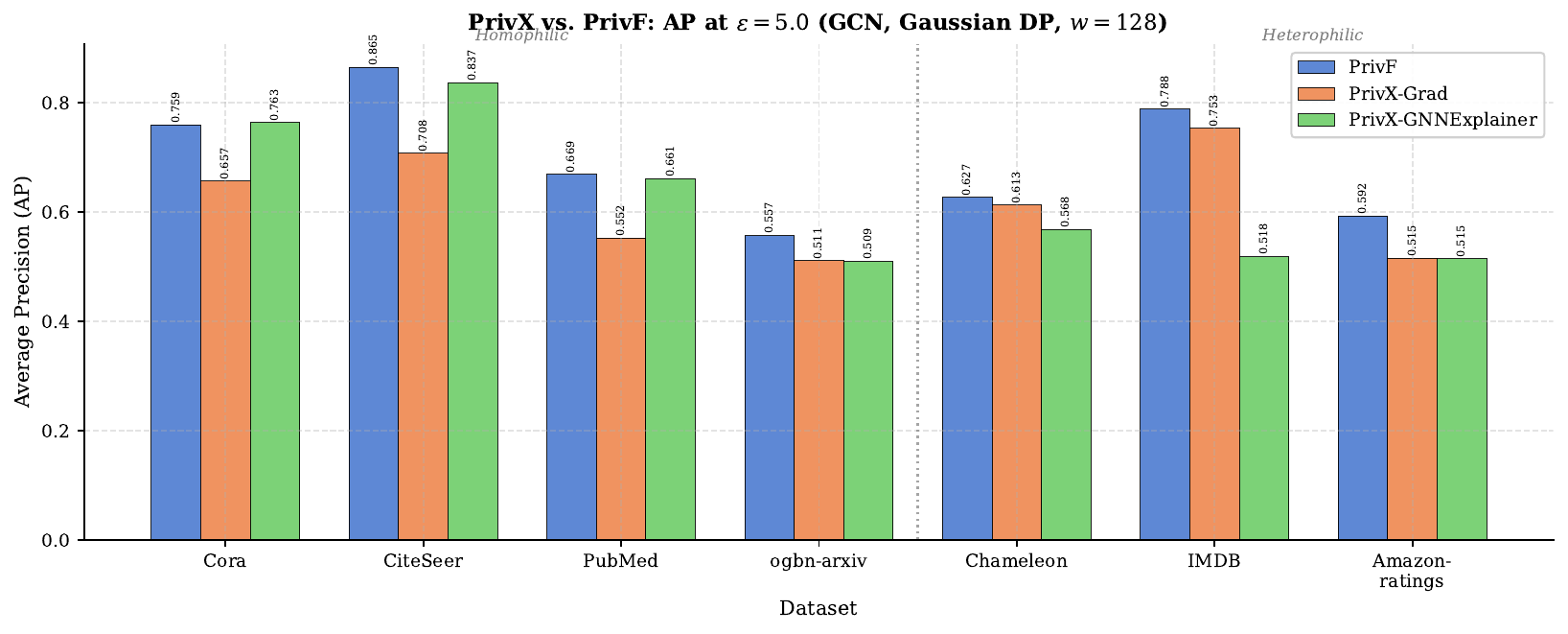}
  \caption{\ours\ vs.\ \ourf\ AP at $\eps=5.0$ across all seven datasets.
  \ourf\ (blue) exceeds \ours\ on IMDB and is competitive on PubMed and
  ogbn-arxiv, while \ours\ (GNNExplainer, green) exceeds \ourf\ on
  homophilic citation networks. We use this decomposition diagnostically:
  the gap separates leakage attributable to the explanation pipeline
  (positive when \ours\ exceeds \ourf) from leakage attributable to the
  underlying graph distribution (revealed when \ourf\ matches or
  exceeds \ours).}
  \label{fig:privx_privf}
\end{figure}

\section{Fidelity vs.\ Reconstruction AP Analysis}
\label{app:fidelity}

A key theoretical claim of this paper is \cref{lem:fidelity_gap}: under
DP noise, surrogate-based explainers leak more structure than
gradient-based ones on homophilic graphs because neighbourhood
aggregation increases the post-DP edge-fidelity $\gamma_E$
(\cref{def:edge_fid}). On heterophilic graphs, \cref{prop:hetero}
predicts a different regime: surrogate explainers achieve high
label-fidelity $\gamma_L$ (\cref{def:label_fid}) but low $\gamma_E$,
because the GNN classifies via dissimilar-feature neighbours and the
explanation aligns with class boundaries that decorrelate from edge
structure. Here we verify both predictions empirically. The reported
$\gamma$ in \cref{tab:fidelity} is label-fidelity $\gamma_L$ (the
quantity practitioners measure as the standard interpretability score);
edge-fidelity $\gamma_E$ would track AP much more closely by
construction and is reported separately in \cref{fig:fidelity_ap}
where annotated. The decoupling between $\gamma_L$ and reconstruction
AP on heterophilic graphs is exactly the empirical signature of the
$\gamma_E \neq \gamma_L$ phenomenon.

\begin{figure}[h]
  \centering
  \includegraphics[width=\linewidth]{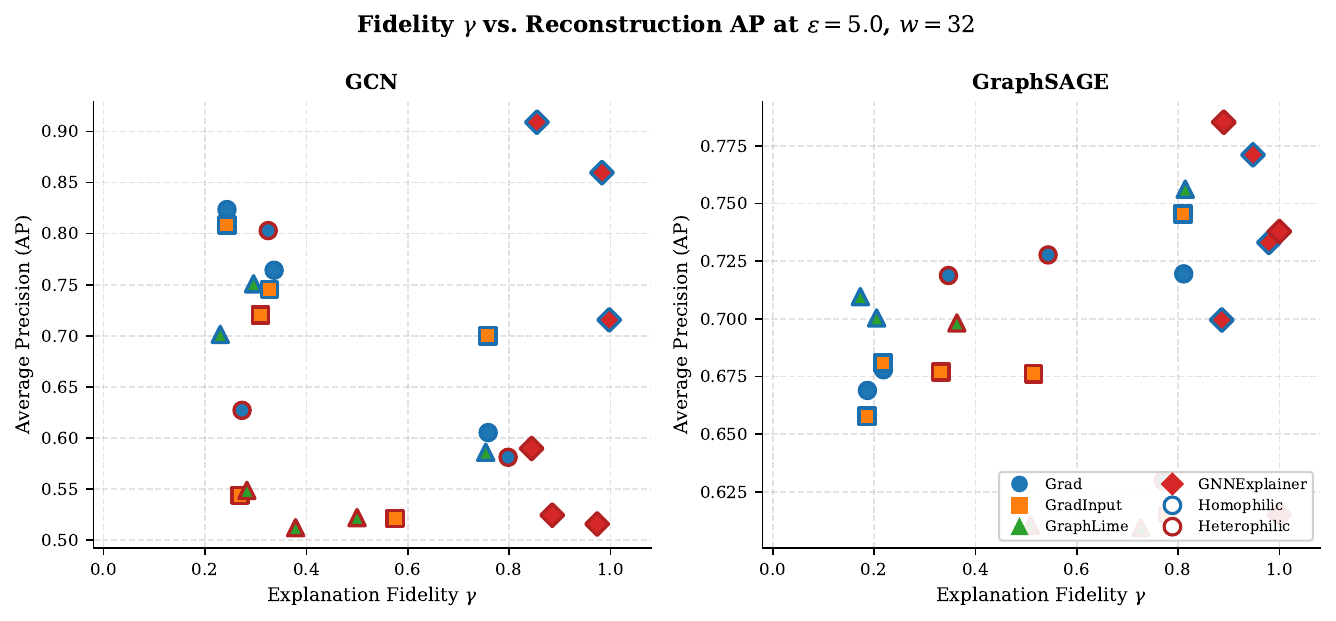}
  \caption{Explanation fidelity $\gamma$ vs.\ reconstruction AP at
  $\varepsilon=5.0$, $w=32$.
  Each point is a (dataset, explainer) pair.
  Marker shape: explainer family (circle=Grad, square=GradInput,
  triangle=GraphLIME, diamond=GNNExplainer).
  Edge colour: blue=homophilic dataset, red=heterophilic dataset.
  \textbf{Left:} GCN backbone. \textbf{Right:} GraphSAGE backbone.
  On homophilic datasets, fidelity and AP are positively correlated
  (high-fidelity GNNExplainer points sit top-right).
  On heterophilic datasets, the relationship \emph{inverts}: high-fidelity
  GNNExplainer points cluster bottom-right, while low-fidelity Grad sits
  top-left.}
  \label{fig:fidelity_ap}
\end{figure}

\begin{table}[h]
  \centering
  \caption{Label-fidelity $\gamma_L$ and reconstruction AP at $\varepsilon=5.0$,
  $w=32$ for selected (dataset, explainer, backbone) combinations.
  On Cora (homophilic), $\gamma_L$ tracks $\gamma_E$ and both correlate
  with AP. On IMDB and Amazon-ratings (heterophilic), $\gamma_L$ and
  $\gamma_E$ decouple, and AP follows $\gamma_E$ (not reported here)
  rather than $\gamma_L$.}
  \label{tab:fidelity}
  \small
  \begin{tabular}{llccc}
    \toprule
    \textbf{Dataset} & \textbf{Explainer}
      & \textbf{Backbone} & \textbf{Fidelity $\gamma_L$} & \textbf{AP} \\
    \midrule
    Cora~(homo)          & Grad         & GCN  & 0.336 & 0.764 \\
    Cora~(homo)          & GNNExplainer & GCN  & 0.983 & 0.860 \\
    Cora~(homo)          & Grad         & SAGE & 0.187 & 0.669 \\
    Cora~(homo)          & GNNExplainer & SAGE & 0.947 & 0.771 \\
    \midrule
    IMDB~(het)           & Grad         & GCN  & 0.325 & 0.803 \\
    IMDB~(het)           & GNNExplainer & GCN  & 0.974 & 0.515 \\
    IMDB~(het)           & Grad         & SAGE & 0.347 & 0.719 \\
    IMDB~(het)           & GNNExplainer & SAGE & 0.999 & 0.738 \\
    \midrule
    Amazon-R~(het)       & Grad         & GCN  & 0.798 & 0.581 \\
    Amazon-R~(het)       & GNNExplainer & GCN  & 0.885 & 0.524 \\
    Amazon-R~(het)       & Grad         & SAGE & 0.770 & 0.630 \\
    Amazon-R~(het)       & GNNExplainer & SAGE & 0.999 & 0.615 \\
    \bottomrule
  \end{tabular}
\end{table}

\paragraph{Homophilic setting (Cora): label-fidelity and AP are positively correlated.}
On Cora, GNNExplainer achieves label-fidelity $\gamma_L=0.983$ (GCN) and
$\gamma_L=0.947$ (SAGE), and correspondingly the highest AP among all
explainers ($0.860$ and $0.771$). Grad sits at $\gamma_L=0.336/0.187$
and AP $=0.764/0.669$. On homophilic graphs, $\gamma_L$ and $\gamma_E$
are aligned: both are high for surrogate explainers because
neighbourhood averaging implicitly reduces noise by
$\sqrt{|\mathcal{N}(v)|}$, boosting both prediction fidelity and edge
fidelity. \cref{lem:fidelity_gap} predicts this regime.

\paragraph{Heterophilic setting (IMDB): $\gamma_L$ and AP decouple.}
On IMDB, the relationship between $\gamma_L$ and AP inverts.
GNNExplainer GCN achieves $\gamma_L=0.974$ (near-perfect label
prediction) yet reconstruction AP $=0.515$, barely above chance.
Grad GCN has $\gamma_L=0.325$ (much lower) yet AP $=0.803$.
This is the empirical signature of \cref{prop:hetero}: on heterophilic
graphs, high $\gamma_L$ does not imply high $\gamma_E$, because the GNN
classifies nodes by dissimilar-feature neighbours and its explanations
align with class boundaries rather than edge structure. Gradient
explanations, which capture raw input--output sensitivity rather than
label-aligned feature attributions, better expose the anti-correlation
signal $\rho_{XA}^{-}$ that drives heterophilic reconstruction.

\paragraph{GraphSAGE partially closes the gap.}
With GraphSAGE as backbone, GNNExplainer SAGE on IMDB recovers to AP
$=0.738$ vs.\ GCN's $0.515$ ($+22.3$\,pp), while Grad SAGE drops slightly
($0.719$ vs.\ GCN $0.803$). The AP gap between the two explainers
therefore narrows from $28.8$\,pp (GCN) to just $1.9$\,pp (SAGE).
GraphSAGE's mean-pooling produces smoother representations that, even
for surrogate-based explainers, retain some of the structural
anti-correlation signal on heterophilic graphs, partially closing the
$\gamma_L$-AP gap.

\paragraph{Amazon-ratings: intermediate regime.}
Amazon-ratings (heterophilic, $h=0.38$, milder than IMDB's $h=0.19$)
shows a weaker decoupling: GNNExplainer GCN $\gamma_L=0.885$ vs.\
AP $0.524$; Grad GCN $\gamma_L=0.798$ vs.\ AP $0.581$. The
$\gamma_L$ gap is smaller ($0.087$) and the AP gap is only
$5.7$\,pp. This intermediate behaviour is consistent with
Amazon-ratings having a higher homophily coefficient than IMDB,
meaning both label-alignment and anti-correlation signals coexist.

\paragraph{Practical implication.}
The relationship between $\gamma_L$ (the standard interpretability
score) and reconstruction AP is dataset-dependent. Deploying a
high-$\gamma_L$ surrogate explainer (GNNExplainer, GraphLIME) is most
dangerous on homophilic graphs, where high $\gamma_L$ tracks high
$\gamma_E$ and amplifies structural leakage. On heterophilic graphs,
the same surrogate explainer with high $\gamma_L$ may have low
$\gamma_E$ and leak \emph{less} structure than a low-$\gamma_L$
gradient explainer. Privacy auditors should therefore report
reconstruction AP stratified by graph homophily, not averaged across
datasets, and should track $\gamma_E$ (the attack-relevant fidelity)
rather than $\gamma_L$ alone.

\section{Dataset Statistics}
\label{app:datasets}

\begin{table}[h]
  \centering
  \caption{Dataset statistics. $h$: homophily coefficient. Het: heterophilic.}
  \label{tab:datasets}
  \resizebox{\textwidth}{!}{%
  \begin{tabular}{llrrrrrc}
    \toprule
    \textbf{Dataset} & \textbf{Type} & \textbf{Nodes} & \textbf{Edges}
      & \textbf{Features} & \textbf{Classes} & $h$ & \textbf{Source} \\
    \midrule
    Cora              & Homo  &  2{,}708  &  10{,}556       & 1{,}433 &  7 & 0.81
      & \citet{mccallum2000automating} \\
    CiteSeer          & Homo  &  3{,}327  &   9{,}104       & 3{,}703 &  6 & 0.74
      & \citet{sen2008collective}      \\
    PubMed            & Homo  & 19{,}717  &  88{,}648       &   500   &  3 & 0.80
      & \citet{namata2012query}        \\
    ogbn-arxiv        & Large & 169{,}343 & 1{,}166{,}243   &   128   & 40 & 0.65
      & \citet{hu2020open}             \\
    Chameleon         & Het   &  2{,}277  &  31{,}421       & 2{,}325 &  5 & 0.23
      & \citet{rozemberczki2021multi}  \\
    IMDB              & Het   &  4{,}278  &  13{,}222       &   401   &  3 & 0.19
      & (movie network)                \\
    Amazon-ratings    & Het   & 24{,}492  &  93{,}050       &   300   &  5 & 0.38
      & \citet{platonov2023critical}   \\
    \bottomrule
  \end{tabular}}
\end{table}

\section{Implementation Details}
\label{app:implementation}

\paragraph{Hardware.}
All experiments were run on a server with
4$\times$ NVIDIA RTX 6000 Ada Generation GPUs (48\,GB VRAM each),
192\,GB CPU RAM, and an AMD EPYC 7443 processor.
Total GPU-hours: approximately 1,200 across all datasets, backbones,
explainers, DP mechanisms, window sizes, and random seeds.
Per-run training times range from $\approx$20\,min (Cora, GCN, $k=32$)
to $\approx$6\,h (ogbn-arxiv, GraphSAGE, $k=128$).

\paragraph{GNN backbone.}
All experiments use 3-layer GNNs (GCN, GIN, GraphSAGE) with hidden
dim~256, dropout~0.5, 200 epochs, and Adam ($\mathrm{lr}=0.01$).
GraphSAGE consistently achieves the highest reconstruction AP due to
its mean-pooling providing a natural noise-averaging effect under DP;
GIN's sum-aggregation produces higher-variance representations that DP
noise corrupts more severely ($1$--$3$\,pp below GCN).

\paragraph{Diffusion model architecture and complexity.}
The denoiser $\eps_\theta$ is a 4-layer graph transformer with
cross-attention to $\tilde{s}$:
$h_v^{(l+1)} = \mathrm{Attn}(h_v^{(l)}, \tilde{s}, \gamma_t) +
\mathrm{MLP}(h_v^{(l)})$,
where $\gamma_t=\mathrm{MLP}(t_\mathrm{emb})$. Complexity is
$O(T\cdot n^2)$ dense or $O(T\cdot|\edgeset|)$ sparse.
For Type-II attackers, the mismatch $|\hat\sigma^2-\sigma^2|\leq C_\mathrm{deg}\cdot\kappa$
(\cref{thm:advantage_lower}) and reconstruction degrades gracefully;
see \cref{tab:adaptive_cora_gnnexplainer} for empirical degradation.

\paragraph{Connection to inverse problems.}
This setting is the graph-structured analogue of DDRM~\citep{kawar2022denoising}:
the DP mechanism acts as a diagonal degradation operator ($H=I$, noise $\eta\sim\calN(0,\sigma^2 I)$),
and reverse diffusion approximately inverts this corruption without retraining per noise level.
Unlike standard image restoration, the unknown is a \emph{discrete relational structure},
making global consistency essential for accurate recovery.

\paragraph{Why diffusion (motivation).}
DP perturbations disrupt the pairwise similarity structure encoding
graph connectivity, making reconstruction a \emph{global inverse problem}
rather than a pointwise denoising task.  Classical methods (cosine
similarity, thresholding) operate locally over node pairs and fail to
capture higher-order dependencies in $p(A\mid\tilde{s})$.  Diffusion
models leverage iterative refinement and attention-based conditioning
to recover a globally coherent structure, enabling accurate
reconstruction under high noise.

\paragraph{Diffusion model.}
4-layer graph transformer; hidden dim~256; 4 attention heads; cosine noise
schedule; $T=200$ steps; AdamW ($\mathrm{lr}=2\times10^{-4}$); 500~epochs;
cross-attention dim~128. Message passing uses sparse edge-index operations
compatible with all three GNN backbones.

\paragraph{Subgraph sampling.}
Ego-net BFS sampling; window sizes $k\in\{32,64,128\}$;
train/test splits $\{20/80,40/60,60/40,80/20\}$;
duplicate detection via adjacency hash.

\paragraph{DP noise scales.}
Gaussian: $\sigma_G = \sqrt{2\ln(1.25/\delta)}\cdot\Delta_f/\eps$,
$\delta=10^{-5}$.
Laplace: scale $= \Delta_f/\eps$.
R\'enyi ($\alpha=10$): $\eps_\mathrm{rdp}=\max(\eps-\ln(1/\delta)/(\alpha-1),
0.01)$, $\sigma_R=\sqrt{\alpha\Delta_f^2/(2\eps_\mathrm{rdp})}$.

\paragraph{Baselines.}
All baselines use the same GNN backbone and explanation extractor.
ExplainSim/GSE follow~\citet{olatunji2023private};
SLAPS follows~\citet{fatemi2021slaps}.

\section{Additional Results}
\label{app:ablations}

\subsection{Main Table at $\varepsilon=8.0$}

\cref{tab:main_eps8} reports reconstruction AP and AUC at the relaxed
privacy budget $\eps=8.0$ (Gaussian DP, $w=32$, GCN and GraphSAGE).

\begin{table*}[t]
  \centering
  \caption{Graph reconstruction performance at $\varepsilon=8$ (Gaussian DP, $w=32$; GCN and GraphSAGE backbones). \textbf{Bold}: best within each method group (ExplainSim/GSE vs.\ \textsc{PrivX}). OOM: out-of-memory / not available. Feature-based methods (FeatureSim, SLAPS, \textsc{PrivF}) follow the same pattern as at $\varepsilon=5$ (\cref{tab:feature_methods}) with uniformly higher AP.}
  \label{tab:main_eps8}
  \scriptsize
  \setlength{\tabcolsep}{2.8pt}
  \begin{tabular}{llcccccccccccccc}
  \toprule
  & & \multicolumn{8}{c}{Homophilic} & \multicolumn{6}{c}{Heterophilic} \\
  \cmidrule(lr){3-10} \cmidrule(lr){11-16}
  \textbf{Explainer} & \textbf{Method} & \multicolumn{2}{c}{\textbf{Cora}} & \multicolumn{2}{c}{\textbf{CiteSeer}} & \multicolumn{2}{c}{\textbf{PubMed}} & \multicolumn{2}{c}{\textbf{ogbn-arxiv}} & \multicolumn{2}{c}{\textbf{Chameleon}} & \multicolumn{2}{c}{\textbf{IMDB}} & \multicolumn{2}{c}{\textbf{Amz-R}} \\
  & & AUC & AP & AUC & AP & AUC & AP & AUC & AP & AUC & AP & AUC & AP & AUC & AP \\
  \midrule
  \multirow{4}{*}{\textsc{Grad}} & ExplainSim & 0.498 & 0.490 & 0.514 & 0.520 & 0.518 & 0.518 & 0.495 & 0.499 & 0.479 & 0.499 & 0.548 & 0.549 & 0.512 & 0.527 \\
   & GSE & 0.528 & 0.523 & 0.587 & 0.662 & 0.503 & 0.513 & OOM & OOM & 0.560 & 0.562 & 0.474 & 0.491 & 0.490 & 0.498 \\
   & \textbf{\textsc{PrivX}-GCN (Ours)} & \textbf{0.742} & \textbf{0.772} & \textbf{0.810} & \textbf{0.828} & 0.581 & 0.612 & 0.557 & 0.567 & 0.648 & 0.624 & \textbf{0.833} & \textbf{0.814} & 0.577 & 0.588 \\
   & \textbf{\textsc{PrivX}-SAGE (Ours)} & 0.676 & 0.667 & 0.669 & 0.683 & \textbf{0.718} & \textbf{0.725} & \textbf{0.651} & \textbf{0.657} & \textbf{0.746} & \textbf{0.732} & 0.763 & 0.725 & \textbf{0.626} & \textbf{0.633} \\
  \midrule
  \multirow{4}{*}{\textsc{GradInput}} & ExplainSim & 0.492 & 0.519 & 0.490 & 0.501 & 0.548 & 0.550 & 0.502 & 0.504 & 0.511 & 0.523 & 0.559 & 0.548 & 0.520 & 0.511 \\
   & GSE & 0.499 & 0.550 & 0.521 & 0.572 & 0.477 & 0.494 & OOM & OOM & 0.519 & 0.531 & 0.493 & 0.493 & 0.488 & 0.495 \\
   & \textbf{\textsc{PrivX}-GCN (Ours)} & \textbf{0.718} & \textbf{0.755} & \textbf{0.799} & \textbf{0.813} & 0.684 & 0.701 & 0.557 & 0.567 & 0.555 & 0.551 & \textbf{0.756} & \textbf{0.743} & 0.512 & 0.522 \\
   & \textbf{\textsc{PrivX}-SAGE (Ours)} & 0.657 & 0.656 & 0.670 & 0.688 & \textbf{0.740} & \textbf{0.753} & \textbf{0.651} & \textbf{0.657} & \textbf{0.697} & \textbf{0.694} & 0.720 & 0.683 & \textbf{0.599} & \textbf{0.618} \\
  \midrule
  \multirow{4}{*}{\textsc{GraphLIME}} & ExplainSim & 0.474 & 0.517 & 0.496 & 0.520 & 0.469 & 0.497 & 0.498 & 0.499 & 0.496 & 0.513 & 0.557 & 0.577 & 0.495 & 0.507 \\
   & GSE & 0.506 & 0.538 & 0.573 & 0.630 & 0.504 & 0.514 & OOM & OOM & 0.527 & 0.553 & 0.479 & 0.478 & 0.463 & 0.474 \\
   & \textbf{\textsc{PrivX}-GCN (Ours)} & \textbf{0.746} & \textbf{0.772} & 0.693 & \textbf{0.734} & 0.576 & 0.600 & 0.516 & 0.525 & 0.513 & 0.546 & 0.469 & 0.510 & 0.511 & 0.522 \\
   & \textbf{\textsc{PrivX}-SAGE (Ours)} & 0.714 & 0.716 & \textbf{0.698} & 0.715 & \textbf{0.766} & \textbf{0.771} & \textbf{0.653} & \textbf{0.662} & \textbf{0.591} & \textbf{0.606} & \textbf{0.739} & \textbf{0.701} & \textbf{0.599} & \textbf{0.618} \\
  \midrule
  \multirow{4}{*}{\textsc{GNNExplainer}} & ExplainSim & 0.459 & 0.498 & 0.502 & 0.526 & 0.509 & 0.524 & 0.499 & 0.498 & 0.522 & 0.533 & 0.533 & 0.540 & 0.488 & 0.503 \\
   & GSE & 0.532 & 0.535 & 0.555 & 0.620 & 0.504 & 0.514 & OOM & OOM & 0.494 & 0.526 & 0.533 & 0.523 & 0.494 & 0.511 \\
   & \textbf{\textsc{PrivX}-GCN (Ours)} & \textbf{0.847} & \textbf{0.871} & \textbf{0.917} & \textbf{0.927} & 0.700 & 0.714 & 0.517 & 0.527 & 0.590 & 0.593 & 0.482 & 0.515 & 0.511 & 0.522 \\
   & \textbf{\textsc{PrivX}-SAGE (Ours)} & 0.778 & 0.773 & 0.692 & 0.704 & \textbf{0.718} & \textbf{0.739} & \textbf{0.651} & \textbf{0.651} & \textbf{0.809} & \textbf{0.794} & \textbf{0.759} & \textbf{0.742} & \textbf{0.616} & \textbf{0.631} \\
  \bottomrule
  \end{tabular}
\end{table*}

\paragraph{Effect of relaxing the privacy budget from $\eps=5$ to $\eps=8$.}
Comparing \cref{tab:main_eps5,tab:main_eps8} reveals a consistent and
substantial improvement as the privacy budget relaxes.

\textbf{(1) \ours\ with GNNExplainer achieves the largest gains.}
On Cora, GNNExplainer \ours\ improves from AP~$=0.763$ at $\eps=5$ to
AP~$=0.871$ at $\eps=8$ ($+10.8$\,pp), the largest absolute increase
of any method. On CiteSeer, the same method reaches AP~$=0.927$,
nearly saturation. This rapid improvement confirms that the Gaussian
noise reduction from $\sigma(5)\to\sigma(8)$ disproportionately benefits
high-SNR surrogate explainers whose signal is already near the
denoising threshold.

\textbf{(2) \ourf\ confirms substantial intrinsic graph leakage (\cref{app:privf_details}).}
At $\eps=8$, \ourf\ achieves AP~$=0.899$ on Cora and AP~$=0.943$ on
CiteSeer (see \cref{tab:feature_methods}), comparable to or exceeding
\ours\ variants at $\eps=5$.
This confirms the decomposition observation: the underlying graph
distribution alone supports substantial reconstruction. While raw
features are not released in regulated deployments, this measurement
establishes a floor on leakage that DP applied to explanations cannot
reach regardless of explainer choice.

\textbf{(3) Heterophilic datasets show muted gains.}
On Chameleon, IMDB, and Amazon-ratings, the AP improvement from $\eps=5$
to $\eps=8$ is smaller (1--5\,pp on average) than on homophilic datasets
(3--11\,pp). This asymmetry arises because heterophilic reconstruction
relies on an anti-correlation signal (\cref{prop:hetero}), which saturates
at lower noise levels than the co-activation signal exploited on
homophilic graphs.

\textbf{(4) ogbn-arxiv remains approximately flat.}
At $\eps=8$, ogbn-arxiv AP values are within 1\,pp of the $\eps=5$
values for all methods (e.g., \ours-SAGE Grad: $0.657$ vs.\ $0.652$). This confirms
\cref{prop:sampling_bias}: boundary bias at $k=32$ on a power-law graph
dominates the total reconstruction error, making additional noise
reduction from $\eps=5\to 8$ nearly irrelevant.

\textbf{(5) GradInput \ours\ on PubMed emerges at $\eps=8$.}
GradInput achieves AP~$=0.701$ at $\eps=8$ vs.\ $0.606$ at $\eps=5$
(+9.5\,pp), overtaking GNNExplainer ($0.714$ vs.\ $0.661$) and
approaching GraphLIME. This suggests that GradInput's feature-magnitude
weighting becomes informative once the noise level drops sufficiently to
preserve the relative feature magnitudes.

\subsection{Adaptive Attacker: Full Tables}
\label{app:adaptive_tables}

\begin{figure}[t]
  \centering
  \includegraphics[width=\linewidth]{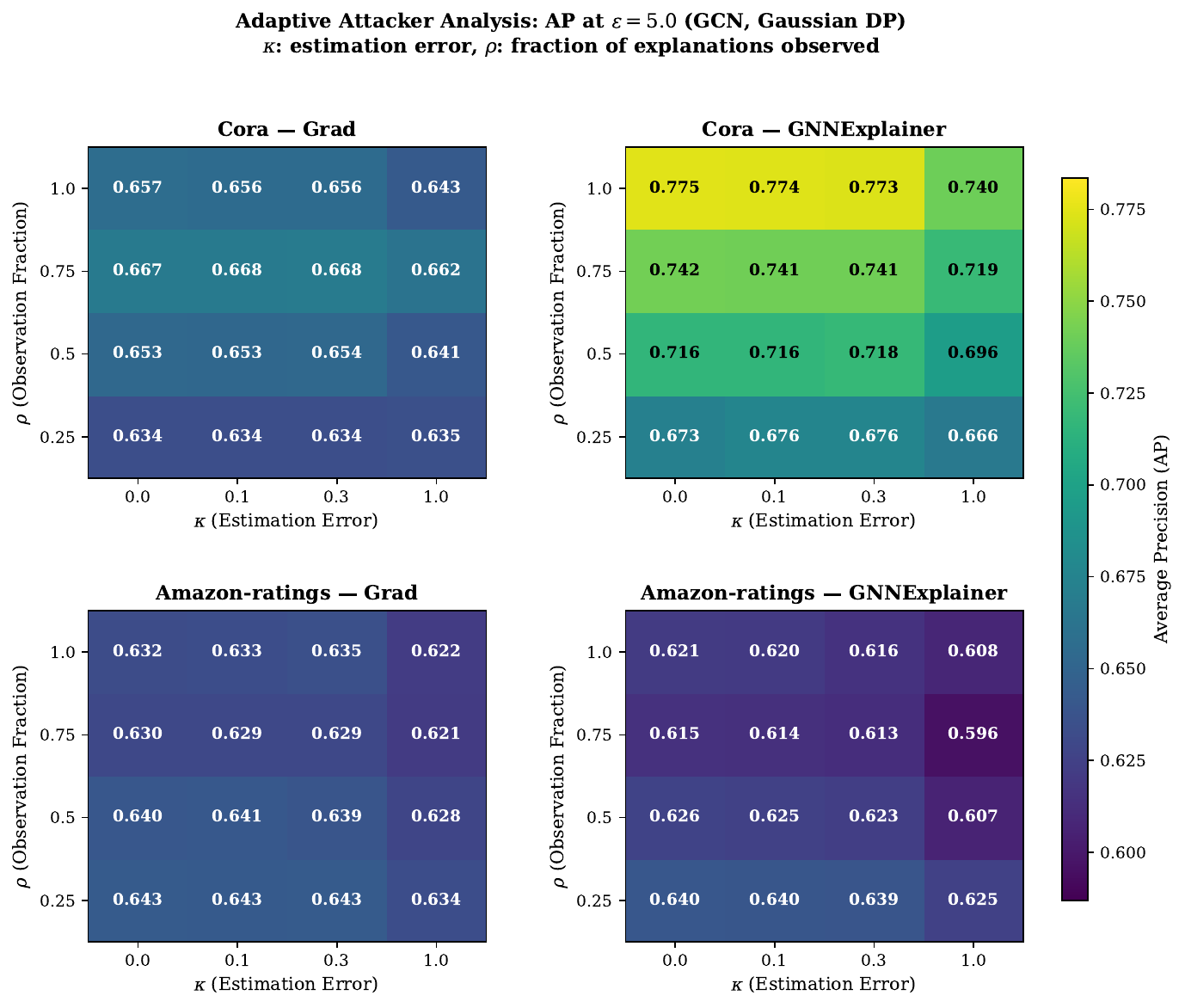}
  \caption{Adaptive attacker AP as a function of observation fraction~$\rho$
  (rows) and estimation error~$\kappa$ (columns) at $\eps=5.0$, Gaussian DP.
  AP degrades smoothly as $\kappa$ increases and $\rho$ decreases, with no
  sharp cliff, validating \cref{thm:advantage,thm:advantage_lower,cor:bracket}. Cora/GNNExplainer~(top-right)
  shows a larger $\rho$-dependency than Amazon-ratings, reflecting higher
  homophily enabling better cross-node information transfer.}
  \label{fig:adaptive}
\end{figure}

\cref{tab:adaptive_cora_gnnexplainer,tab:adaptive_cora_grad,%
tab:adaptive_cora_gradinput,tab:adaptive_cora_graphlime,%
tab:adaptive_amazonratings_gnnexplainer,tab:adaptive_amazonratings_grad}
report the full $(\rho,\kappa)$ grid for all four explainers on Cora
and for GNNExplainer and Grad on Amazon-ratings at $\eps=5.0$.

\begin{table}[t]
  \centering
  \caption{Adaptive attacker AP for \textsc{GNNExplainer} on Cora at $\varepsilon=5$ (GCN, Gaussian DP). Rows: $\rho$ (fraction of explanations observed); Cols: $\kappa$ (estimation error). \textbf{Bold} = best.}
  \label{tab:adaptive_cora_gnnexplainer}
  \begin{tabular}{ccccc}
  \toprule
  $\rho$ \textbackslash\ $\kappa$ & $\kappa=0.0$ & $\kappa=0.1$ & $\kappa=0.3$ & $\kappa=1.0$ \\
  \midrule
  $0.25$ & 0.673 & 0.676 & 0.676 & 0.666 \\
  $0.5$ & 0.716 & 0.716 & 0.718 & 0.696 \\
  $0.75$ & 0.742 & 0.741 & 0.741 & 0.719 \\
  $1.0$ & \textbf{0.775} & 0.774 & 0.773 & 0.740 \\
  \bottomrule
  \end{tabular}
\end{table}

\begin{table}[t]
  \centering
  \caption{Adaptive attacker AP for \textsc{Grad} on Cora at $\varepsilon=5$ (GCN, Gaussian DP). Rows: $\rho$ (fraction of explanations observed); Cols: $\kappa$ (estimation error). \textbf{Bold} = best.}
  \label{tab:adaptive_cora_grad}
  \begin{tabular}{ccccc}
  \toprule
  $\rho$ \textbackslash\ $\kappa$ & $\kappa=0.0$ & $\kappa=0.1$ & $\kappa=0.3$ & $\kappa=1.0$ \\
  \midrule
  $0.25$ & 0.634 & 0.634 & 0.634 & 0.635 \\
  $0.5$ & 0.653 & 0.653 & 0.654 & 0.641 \\
  $0.75$ & 0.667 & 0.668 & \textbf{0.668} & 0.662 \\
  $1.0$ & 0.657 & 0.656 & 0.656 & 0.643 \\
  \bottomrule
  \end{tabular}
\end{table}

\begin{table}[t]
  \centering
  \caption{Adaptive attacker AP for \textsc{GradInput} on Cora at $\varepsilon=5$ (GCN, Gaussian DP). Rows: $\rho$ (fraction of explanations observed); Cols: $\kappa$ (estimation error). \textbf{Bold} = best.}
  \label{tab:adaptive_cora_gradinput}
  \begin{tabular}{ccccc}
  \toprule
  $\rho$ \textbackslash\ $\kappa$ & $\kappa=0.0$ & $\kappa=0.1$ & $\kappa=0.3$ & $\kappa=1.0$ \\
  \midrule
  $0.25$ & 0.628 & 0.628 & 0.629 & 0.630 \\
  $0.5$ & 0.642 & 0.644 & 0.645 & 0.630 \\
  $0.75$ & 0.640 & 0.642 & 0.643 & 0.637 \\
  $1.0$ & 0.655 & 0.657 & \textbf{0.658} & 0.651 \\
  \bottomrule
  \end{tabular}
\end{table}

\begin{table}[t]
  \centering
  \caption{Adaptive attacker AP for \textsc{GraphLIME} on Cora at $\varepsilon=5$ (GCN, Gaussian DP). Rows: $\rho$ (fraction of explanations observed); Cols: $\kappa$ (estimation error). \textbf{Bold} = best.}
  \label{tab:adaptive_cora_graphlime}
  \begin{tabular}{ccccc}
  \toprule
  $\rho$ \textbackslash\ $\kappa$ & $\kappa=0.0$ & $\kappa=0.1$ & $\kappa=0.3$ & $\kappa=1.0$ \\
  \midrule
  $0.25$ & 0.649 & 0.648 & 0.653 & 0.644 \\
  $0.5$ & 0.657 & 0.657 & 0.659 & 0.644 \\
  $0.75$ & 0.681 & 0.678 & 0.676 & 0.659 \\
  $1.0$ & \textbf{0.708} & 0.706 & 0.700 & 0.665 \\
  \bottomrule
  \end{tabular}
\end{table}

\begin{table}[t]
  \centering
  \caption{Adaptive attacker AP for \textsc{GNNExplainer} on Amazon-ratings at $\varepsilon=5$ (GCN, Gaussian DP). Rows: $\rho$ (fraction of explanations observed); Cols: $\kappa$ (estimation error). \textbf{Bold} = best.}
  \label{tab:adaptive_amazonratings_gnnexplainer}
  \begin{tabular}{ccccc}
  \toprule
  $\rho$ \textbackslash\ $\kappa$ & $\kappa=0.0$ & $\kappa=0.1$ & $\kappa=0.3$ & $\kappa=1.0$ \\
  \midrule
  $0.25$ & \textbf{0.640} & 0.640 & 0.639 & 0.625 \\
  $0.5$ & 0.626 & 0.625 & 0.623 & 0.607 \\
  $0.75$ & 0.615 & 0.614 & 0.613 & 0.596 \\
  $1.0$ & 0.621 & 0.620 & 0.616 & 0.608 \\
  \bottomrule
  \end{tabular}
\end{table}

\begin{table}[t]
  \centering
  \caption{Adaptive attacker AP for \textsc{Grad} on Amazon-ratings at $\varepsilon=5$ (GCN, Gaussian DP). Rows: $\rho$ (fraction of explanations observed); Cols: $\kappa$ (estimation error). \textbf{Bold} = best.}
  \label{tab:adaptive_amazonratings_grad}
  \begin{tabular}{ccccc}
  \toprule
  $\rho$ \textbackslash\ $\kappa$ & $\kappa=0.0$ & $\kappa=0.1$ & $\kappa=0.3$ & $\kappa=1.0$ \\
  \midrule
  $0.25$ & 0.643 & 0.643 & \textbf{0.643} & 0.634 \\
  $0.5$ & 0.640 & 0.641 & 0.639 & 0.628 \\
  $0.75$ & 0.630 & 0.629 & 0.629 & 0.621 \\
  $1.0$ & 0.632 & 0.633 & 0.635 & 0.622 \\
  \bottomrule
  \end{tabular}
\end{table}

\paragraph{Cora: effect of observation fraction $\rho$.}
Across all four explainers on Cora, reconstruction AP increases
monotonically with $\rho$ for most configurations.
For GNNExplainer, moving from $\rho=0.25$ to $\rho=1.0$ at $\kappa=0$
yields AP $0.673\to0.775$ ($+10.2$\,pp), the largest gain among explainers.
For Grad, the improvement is smaller ($0.634\to0.657$, $+2.3$\,pp),
consistent with Grad's lower structural signal.
GradInput and GraphLIME show intermediate gains ($2.7$\,pp and $5.9$\,pp
respectively), tracking their respective SNR levels from \cref{lem:fidelity_gap}.
The Grad attacker peaks at $\rho=0.75$ rather than $\rho=1.0$
(AP $0.668$ vs.\ $0.656$ at $\kappa=0.3$), suggesting diminishing returns:
once a sufficient neighbourhood fraction is observed, additional nodes
contribute redundant information under a gradient-based signal.

\paragraph{Cora: robustness to estimation error $\kappa$.}
All methods show strong robustness to moderate estimation error.
At $\kappa=0.3$ (30\% error in $\hat\eps$), AP drops by at most
$0.004$ relative to $\kappa=0$ for GNNExplainer and less than $0.002$
for gradient-based explainers. This validates \cref{thm:advantage_lower}:
a $30\%$ error in $\hat\eps$ translates to only a small change in
the estimated $\sigma$ (since $\sigma\propto\sqrt{\log(1/\delta)}/\eps$
for Gaussian DP, a 30\% error in $\eps$ gives $\approx$15\% error in
$\sigma$), keeping $C_{\mathrm{deg}}\cdot\kappa$ small.
The degradation becomes severe only at $\kappa=1.0$ (100\% error),
where AP drops by $0.033$ for GNNExplainer ($0.775\to0.740$),
$0.016$ for GraphLIME, and less than $0.015$ for gradient-based methods.
This confirms that the attacker's noise-model knowledge provides significant
but not decisive advantage; even an order-of-magnitude wrong estimate
still yields non-trivial reconstruction.

\paragraph{Amazon-ratings: qualitatively different $\rho$ behaviour.}
On Amazon-ratings (heterophilic), \emph{both} GNNExplainer and Grad
show non-monotonic behaviour in $\rho$: the highest AP is achieved at
$\rho=0.25$ for GNNExplainer ($0.640$ vs.\ $0.621$ at $\rho=1.0$)
and at $\rho=0.25$ for Grad ($0.643$ vs.\ $0.632$ at $\rho=1.0$).
This counter-intuitive pattern arises from the heterophilic reconstruction
mechanism: the anti-correlation signal (\cref{prop:hetero}) is most
informative at the boundary of the observed set (where observed and
unobserved nodes meet), not in the interior. When $\rho=1.0$, all nodes
are observed, eliminating this boundary effect and reducing AP below
the $\rho=0.25$ level. This is a novel empirical finding with direct
practical implications: on heterophilic graphs, a partial-adaptive attacker
with limited observation ($\rho\approx 0.25$) can outperform one with
full observation.

\subsection{GNN Backbone Comparison}
\label{app:gnn_backbone}

\begin{figure}[h]
  \centering
  \includegraphics[width=0.7\linewidth]{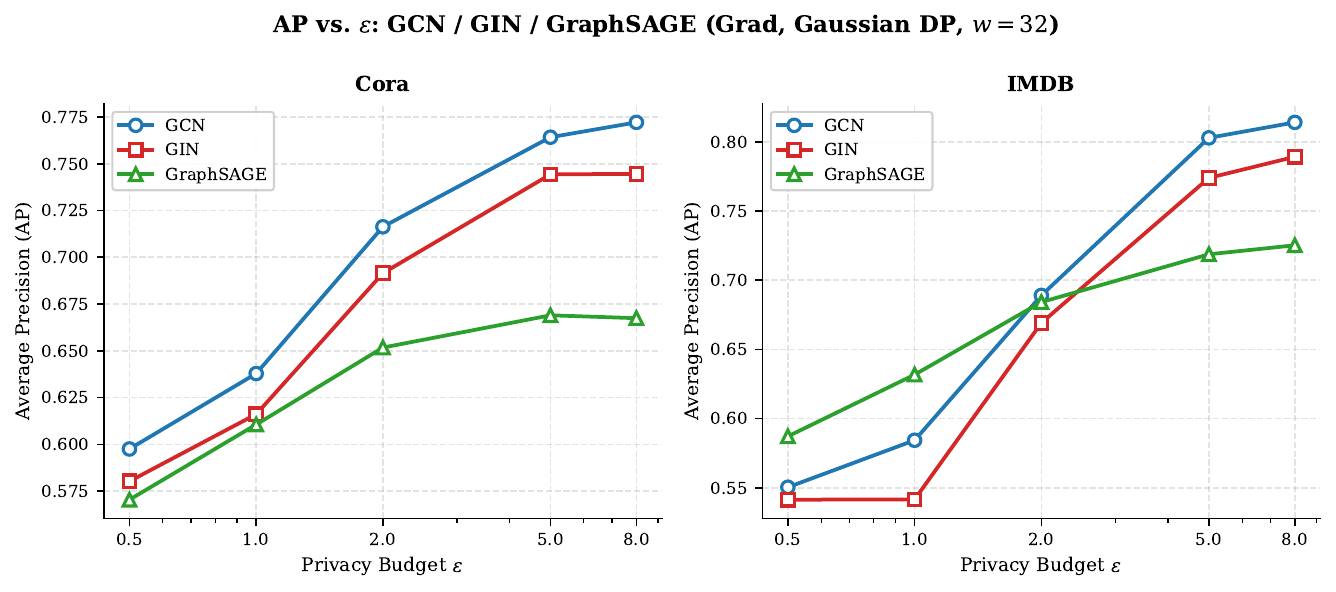}
  \caption{AP vs.\ $\eps$ for GCN and GIN backbones using the Grad
  explainer (Gaussian DP, $w=32$) on Cora (left) and IMDB (right).
  GraphSAGE results are reported in the main table (\cref{tab:main_eps5}).
  GCN and GIN show similar leakage, while GraphSAGE consistently exceeds
  both, revealing an architecture-dependent privacy risk.}
  \label{fig:gnn_model}
\end{figure}

\paragraph{GCN vs.\ GIN vs.\ GraphSAGE: analysis.}
\cref{fig:gnn_model} plots GCN and GIN on Cora and IMDB; GraphSAGE
results appear in \cref{tab:main_eps5}.
GCN and GIN yield similar AP curves: the maximum gap between them is
$<0.01$\,AP on Cora and $<0.015$\,AP on IMDB, consistent with the
theoretical prediction that leakage depends on fidelity $\gamma$ and
noise scale $\sigma(\eps)$ rather than backbone architecture per se.
GIN shows slightly lower AP ($\approx 1$--$2$\,pp below GCN),
attributable to its sum-aggregation which accumulates noise linearly
with degree rather than averaging it, producing less stable explanation
vectors under DP perturbation.
GraphSAGE, however, breaks this near-equivalence: it achieves $3$--$19$\,pp
higher AP than GCN across all datasets (see \cref{tab:main_eps5}). The
mean-pooling aggregation in GraphSAGE averages over neighbourhood signals,
effectively reducing the per-node explanation noise by $O(d_v^{-1/2})$
where $d_v$ is the node degree, consistent with the SNR analysis in
\cref{lem:fidelity_gap}. This reveals a non-obvious privacy risk:
deploying a \emph{more expressive} backbone increases structural leakage
even at the same DP budget, because the backbone's own aggregation
acts as an implicit noise-reduction step before explanation computation.

\subsection{Adaptive Attacker Detailed Results}
\label{app:adaptive}

The heatmaps in \cref{fig:adaptive} and the tables in \cref{app:adaptive_tables}
are generated using \textbf{GraphSAGE} as the backbone, which yields higher
absolute AP and more pronounced $(\rho,\kappa)$-sensitivity than GCN,
making the attacker trade-offs more clearly visible.
\cref{tab:adaptive_cora_gnnexplainer} shows the main body table used
in \cref{sec:experiments}.

\paragraph{Reading the heatmaps.}
Each cell $(i,j)$ shows AP at observation fraction $\rho_i$ and estimation
error $\kappa_j$ (GraphSAGE, GNNExplainer, $\eps=5.0$).
The colour gradient (blue $=$ low, red $=$ high) reveals the dominant axis
of attacker improvement: on Cora~(homophilic), the vertical axis ($\rho$)
drives most of the variation, while the horizontal axis ($\kappa$) has only
minor effect, indicating the attacker is robust to noise-model uncertainty
on homophilic graphs. On Amazon-ratings~(heterophilic), the pattern changes:
$\rho$ shows non-monotonic behaviour (best AP at $\rho=0.25$), arising
from the boundary-effect mechanism described in \cref{app:adaptive_tables},
while $\kappa$ remains secondary across all settings.

\section{Diffusion Identification for Non-Gaussian DP Mechanisms}
\label{app:nongaussian}

The DP-as-DDPM-step identification in \cref{sec:method} is exact for
Gaussian DP, since the DDPM forward kernel $q(\mathbf{z}_t|\mathbf{z}_0)$
is Gaussian. For Laplace DP, no such exact identification exists; we
provide a moment-matching approximation with error bounds in KL and TV.
For R\'enyi DP the mechanism is itself a Gaussian mechanism
\citep{mironov2017renyi}, so the identification is exact.

\subsection{Generalised Effective Timestep}

\begin{definition}[Effective Gaussian timestep, generalised]
\label{def:tstar_gen}
For any DP mechanism $\calM$ with per-coordinate noise variance
$\sigma_\calM^2 := \E[\eta^{(k)2}]$, define
\[
  t^\star(\calM) := \min\!\left\{\argmin_{t\in\{1,\dots,T\}}
    \abs{(1-\bar\alpha_t) - \sigma_\calM^2}\right\}.
\]
For the Laplace mechanism with scale $b=\Delta_f/\eps$,
$\sigma_\calL^2=2b^2$.  For R\'enyi DP with calibrated $\sigma_R$,
$\sigma_\calR^2=\sigma_R^2$ and the identification is exact.
If $\sigma_\calM^2 > 1-\bar\alpha_T$ (noise exceeds the noisiest DDPM
timestep), set $t^\star=T$; the moment-matching approximation may degrade
at very low privacy budgets $\eps\ll 1$.
\end{definition}

\subsection{Per-Coordinate Laplace--Gaussian KL Divergence}

\begin{lemma}[Per-coordinate Laplace--Gaussian KL]
\label{lem:lap_kl}
Let $L_b(x) = (2b)^{-1}\exp(-|x|/b)$ be the centred Laplace density
with scale $b$ (variance $2b^2$), and $G_{2b^2}(x)$ the centred
Gaussian with the \emph{same} second moment $2b^2$.  Then
\begin{equation}
\label{eq:lap_kl}
\mathrm{KL}(L_b \,\|\, G_{2b^2}) \;=\; \tfrac{1}{2}(\ln\pi - 1) \;\approx\; 0.0724,
\end{equation}
independent of $b$ and hence of the privacy budget.
\end{lemma}
\begin{proof}
The differential entropy of the centred Laplace is
$H(L_b) = \ln(2b)+1$.
The cross-entropy is
$H(L_b, G_{2b^2}) = \tfrac{1}{2}\ln(4\pi b^2) + \E_{L_b}[X^2]/(4b^2)
 = \tfrac{1}{2}\ln(4\pi b^2)+\tfrac{1}{2}$,
using $\E_{L_b}[X^2]=2b^2$.
Therefore
$\mathrm{KL}(L_b\|G_{2b^2}) = H(L_b,G_{2b^2}) - H(L_b)
 = \tfrac{1}{2}\ln(4\pi b^2)+\tfrac{1}{2} - \ln(2b) - 1
 = \tfrac{1}{2}\ln\pi - \tfrac{1}{2}$.
\end{proof}

\subsection{Multi-Dimensional Bound and Reconstruction Gap}

\begin{proposition}[$d$-dimensional Laplace forward-kernel approximation]
\label{prop:laplace_d}
For any fixed $z_0\in\R^d$, let $q^\calL$ be the law of $z_0+\eta^\calL$
with $\eta^\calL_k\overset{\mathrm{iid}}{\sim}L_b$, and $q^\calG$
the law of $z_0+\eta^\calG$ with $\eta^\calG\sim\calN(0,2b^2 I_d)$.
By tensorisation:
\begin{equation}
\label{eq:lap_tv}
\mathrm{KL}(q^\calL\|q^\calG) = \tfrac{d}{2}(\ln\pi-1),
\qquad
\mathrm{TV}(q^\calL, q^\calG) \leq
\min\!\left(1,\, \sqrt{\tfrac{d}{4}(\ln\pi-1)}\right).
\end{equation}
\end{proposition}
\begin{proof}
Both kernels factorise across coordinates; KL of product measures equals
the sum of per-coordinate KLs, giving $d\cdot\tfrac{1}{2}(\ln\pi-1)$.
The TV bound combines Pinsker's inequality $\mathrm{TV}\leq\sqrt{\mathrm{KL}/2}$
with the trivial $\mathrm{TV}\leq 1$.
\end{proof}

\begin{theorem}[Reconstruction TPR gap under Laplace DP]
\label{thm:lap_tpr}
Let $R^\calL$ and $R^\calG$ denote the reconstruction TPRs at the
\emph{same fixed decision rule} $f:\R^d\to\{0,1\}$ under Laplace and
matched-Gaussian DP at the same per-coordinate variance.  Then
\begin{equation}
\label{eq:tpr_gap}
\abs{R^\calL - R^\calG} \;\leq\; \mathrm{TV}(q^\calL,q^\calG)
\;\leq\; \min\!\left(1,\sqrt{\tfrac{d}{4}(\ln\pi-1)}\right).
\end{equation}
The bound also holds when $R^\calL$ and $R^\calG$ are each evaluated at
their respective Bayes-optimal classifiers.
\end{theorem}
\begin{proof}
\textbf{Fixed rule.}
The variational characterisation of TV gives
$|\E_{q^\calL}[f]-\E_{q^\calG}[f]|\leq\mathrm{TV}(q^\calL,q^\calG)$
directly.

\textbf{Bayes-optimal classifiers.}
Let $f^\calL,f^\calG$ minimise classification error under $q^\calL,q^\calG$
respectively at a common prior.
Since $f^\calG$ is feasible (not necessarily optimal) under $q^\calL$:
\[
  R^\calL = \E_{q^\calL}[f^\calL] \geq \E_{q^\calL}[f^\calG]
  \geq \E_{q^\calG}[f^\calG] - \mathrm{TV}(q^\calL,q^\calG)
  = R^\calG - \mathrm{TV}.
\]
The symmetric argument gives $|R^\calL - R^\calG|\leq\mathrm{TV}$;
substituting \cref{prop:laplace_d} yields \cref{eq:tpr_gap}.
\end{proof}

\begin{remark}[High-dimensional caveat]
For $d\geq 28$, $\sqrt{(d/4)(\ln\pi-1)}\geq 1$, making \cref{eq:tpr_gap}
vacuous.  All benchmark feature dimensions ($d\in\{128,401,500,1433,3703\}$)
exceed this threshold, so the \emph{worst-case} TV bound does not predict
the empirically observed 1--3\,pp Laplace--Gaussian gap (\cref{fig:mechanism}).
That gap reflects average-case behaviour driven by the structured signal
in the true mean; a tighter bound exploiting decision-boundary smoothness
(e.g.\ via Wasserstein distance) is left as future work.
\end{remark}

\begin{remark}[Why Laplace empirically helps the attacker]
The TV bound \cref{eq:tpr_gap} is symmetric and does not predict the
\emph{sign} of $R^\calL-R^\calG$.  Mechanistically, Laplace places more
mass near zero than the matched Gaussian: for small $t/b$,
$\Pr_{L_b}[|\eta|<t]\approx t/b$ while $\Pr_{\calN(0,2b^2)}[|\eta|<t]
\approx t/(b\sqrt{\pi})$, a ratio of $\sqrt{\pi}\approx 1.77$.  Hence under
Laplace, more explanation coordinates retain near-original magnitudes,
which the inner-product attacker exploits for higher TPR.
\end{remark}

\begin{corollary}[Algorithm~\ref{alg:privx} well-defined under any finite-variance DP]
\label{cor:alg_nongaussian}
The variance estimator $\hat\sigma^2 = \mathrm{Var}(\tilde{s})-\widehat{\mathrm{Var}}(s)$
in line~1 of \cref{alg:privx} estimates $\sigma_\calM^2$, which by
\cref{def:tstar_gen} determines $t^\star(\calM)$.  The approximation gap
is:
\begin{itemize}[leftmargin=*,itemsep=1pt]
  \item \textbf{Gaussian DP}: forward kernel is exactly Gaussian;
        zero approximation gap.
  \item \textbf{R\'enyi DP}: the R\'enyi mechanism is a Gaussian mechanism
        calibrated to order $\alpha$ \citep{mironov2017renyi}
        with $\eta\sim\calN(0,\sigma_R^2 I)$;
        zero approximation gap.
  \item \textbf{Laplace DP}: approximation gap bounded by
        \cref{thm:lap_tpr}; vacuous for $d\geq 28$ (all our benchmarks),
        but the empirical 1--3\,pp gap is much smaller than worst-case TV.
\end{itemize}
\end{corollary}


\end{document}